\documentclass[a4paper,conference]{IEEEtran}
\usepackage{graphicx}
\usepackage{amsfonts}
\usepackage{caption}
\usepackage[labelformat=simple]{subcaption}
\usepackage{float}
\usepackage{amssymb}
\usepackage{enumitem}
\usepackage{algorithm} 
\usepackage{algorithmic}
\usepackage[algo2e]{algorithm2e}
\usepackage{amsmath,amsthm,mathtools}
\usepackage{url}
\usepackage{stmaryrd}
\usepackage{cite}
\usepackage{multirow}
\usepackage[table,xcdraw]{xcolor}
\usepackage{url}
\usepackage{ragged2e}
\usepackage{epstopdf}
\usepackage{eso-pic}

\newcommand{\etal}{\textit{et al}. }
\newcommand{\ie}{\textit{i}.\textit{e}., }

\newcommand\AtPageUpperMyright[1]{\AtPageUpperLeft{%
 \put(\LenToUnit{0.5\paperwidth},\LenToUnit{-1cm}){%
     \parbox{0.5\textwidth}{\raggedleft\fontsize{9}{11}\selectfont #1}}%
 }}%
\newcommand{\conf}[1]{%
\AddToShipoutPictureBG*{%
\AtPageUpperMyright{#1}
}
}

\setlist{nolistsep}

\DeclareMathOperator{\diag}{\operatorname{diag}}
\DeclareMathOperator{\spanmath}{\operatorname{span}}
\DeclareMathOperator{\rankmath}{\operatorname{rank}}

\DeclareMathOperator{\detmath}{\operatorname{det}}

\DeclareMathOperator*{\argmin}{\operatorname{arg\,min}}

\newtheorem{theorem}{Theorem}
\newtheorem{lemma}[theorem]{Lemma}

\newtheorem{definition}{Definition}

\ifCLASSINFOpdf
\else
\fi
\begin{document}
%

\title{GraphBGS: Background Subtraction via Recovery of Graph Signals}

\author{\IEEEauthorblockN{Jhony H. Giraldo}
\IEEEauthorblockA{Laboratoire MIA, La Rochelle Université\\
La Rochelle, France\\
Email: jgiral01@univ-lr.fr}
\and
\IEEEauthorblockN{Thierry Bouwmans}
\IEEEauthorblockA{Laboratoire MIA, La Rochelle Université\\
La Rochelle, France\\
Email: tbouwman@univ-lr.fr}}


%


\maketitle
\conf{\textbf{To be published in 2020 25th International Conference on Pattern Recognition (ICPR), Milan, Italy, January 10-15, 2021}}

\begin{abstract}
Background subtraction is a fundamental pre-processing task in computer vision. This task becomes challenging in real scenarios due to variations in the background for both static and moving camera sequences. Several deep learning methods for background subtraction have been proposed in the literature with competitive performances. However, these models show performance degradation when tested on unseen videos; and they require huge amount of data to avoid overfitting. Recently, graph-based algorithms have been successful approaching unsupervised and semi-supervised learning problems. Furthermore, the theory of graph signal processing and semi-supervised learning have been combined leading to new insights in the field of machine learning. In this paper, concepts of recovery of graph signals are introduced in the problem of background subtraction. We propose a new algorithm called Graph BackGround Subtraction (GraphBGS), which is composed of: instance segmentation, background initialization, graph construction, graph sampling, and a semi-supervised algorithm inspired from the theory of recovery of graph signals. Our algorithm has the advantage of requiring less labeled data than deep learning methods while having competitive results on both: static and moving camera videos. GraphBGS outperforms unsupervised and supervised methods in several challenging conditions on the publicly available Change Detection (CDNet2014), and UCSD background subtraction databases.
\end{abstract}


%
\IEEEpeerreviewmaketitle

\section{Introduction}

Background subtraction is an important topic in computer vision and video analysis. There is a large number of applications such as automatic surveillance of human activities in public spaces, intelligent transportation, industrial vision, among others \cite{garcia2020background}. Background subtraction aims to separate the moving objects from the static scene called background; for example it is commonly accepted that background subtraction includes foreground-background classification and moving object detection \cite{bouwmans2014traditional,bouwmans2017decomposition}. In the literature, there have been scientific efforts trying to improve background subtraction methods in complex scenarios, but no method is able to simultaneously address all the challenges in real cases \cite{tezcan2020bsuv,wang2014cdnet}. As a matter of fact, several papers have focused on designing models for specific challenges in background subtraction \cite{oreifej2012simultaneous,pham2014detection,li2018fusion}.

Background subtraction methods can be classified as unsupervised and supervised learning. Unsupervised methods usually rely on background models to predict the foreground, however these methods have problems when applied on complex scenarios \cite{bouwmans2014traditional,bouwmans2014robust}. On the other hand, supervised methods use deep learning networks to predict the foreground \cite{bouwmans2019deep}. Nevertheless, some of these deep learning methods have problems when the models are evaluated on different video sequences in which the neural networks were trained (unseen videos) \cite{tezcan2020bsuv}. For instance, the performance of FgSegNet v2 \cite{lim2019learning} drops dramatically when evaluated on unseen videos as shown in \cite{tezcan2020bsuv}. Moreover, there are still several open issues on background subtraction: 1) neither unsupervised or supervised methods can effectively handle static as well as moving camera sequences, 2) deep learning methods lack of theoretical guarantees, 3) some deep learning background subtraction methods have not documented their results on unseen videos, or their performances are not good in this evaluation setting \cite{choo2018multi,tezcan2020bsuv}, 4) there is not a theoretical answer about the sample complexity required in deep learning regimen \cite{vapnik2013nature,shalev2014understanding}, 5) currently, most of the background subtraction methods in deep learning involve similar images in training and evaluation, \ie they are not evaluated on unseen videos. As a consequence, a real evaluation of the level of overfitting in these algorithms does not exist, particularly when the amount of available information is not huge (as it is the case in current background subtraction databases). In this paper, the problem of background subtraction is addressed using concepts of the emerging field of Graph Signal Processing (GSP).

\begin{figure*}
    \centering
    \includegraphics[width=0.94\textwidth]{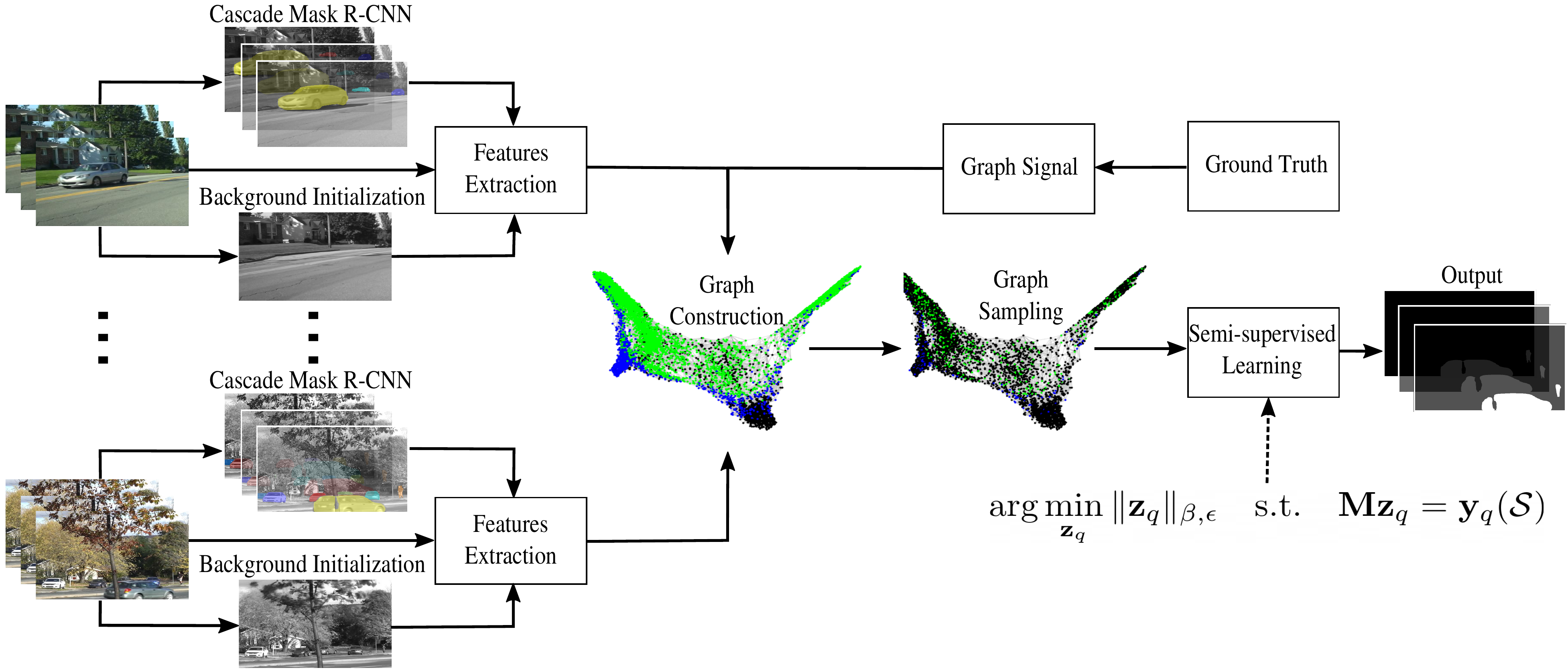}
    \caption{Pipeline of GraphBGS. This algorithm uses temporal median filter as background initialization, and Cascade Mask Region Convolutional Neural Network (R-CNN) \cite{cai2019cascade} to obtain the instances of the videos. Each instance represents a node in the graph, and the representation of each node is obtained with optical flow, intensity, and texture features. Green and blue nodes correspond to moving and static objects, respectively. Black nodes correspond to unknown instances in the database. Finally, some nodes are sampled and the semi-supervised algorithm reconstructs the non-sampled labels in the graph.}
    \label{fig:pipeline}
\end{figure*}

Graphs provide the ability to model interactions of data residing on irregular and complex structures. Social, financial, and sensor networks are examples of data that can be modeled on graphs. GSP is thus an emerging field that extends the concepts of classical digital signal processing to signals supported on graphs \cite{shuman2013emerging,ortega2018graph}, where graph signals are defined as functions over the nodes of a graph. The sampling and recovery of graph signals are fundamental tasks in GSP that have received considerable attention recently \cite{chen2015discrete,anis2016efficient,parada2019blue}. In fact, the problem of semi-supervised learning can be modeled as the reconstruction of a graph signal from its samples \cite{anis2016efficient}.

This work explores a radical departure from prior work by treating background subtraction as a problem of reconstruction of a graph signal from its samples. We propose a new algorithm called Graph BackGround Subtraction (GraphBGS). This algorithm models the instances in videos as nodes of a graph embedded in a high dimensional space, and a graph signal is related to the background or foreground. GraphBGS is composed of instance segmentation, background initialization, construction of a graph, sampling of graph signals, and a reconstruction algorithm to classify if an instance is a moving or a static object. GraphBGS outperforms state-of-the-art methods in several challenging conditions on the Change Detection 2014 (CDNet2014) \cite{wang2014cdnet} and UCSD background subtraction \cite{mahadevan2009spatiotemporal} databases.

The main contribution of this paper are summarized as follows:
\begin{itemize}
    \item Concepts of GSP are introduced for the first time in background subtraction.
    \item The condition number of a perturbed Laplacian matrix is bounded in Theorem \ref{trm:perturbation_matrix}. These bounds show the existence and stability of the inverse of a Sobolev norm in GSP.
    \item The proposed algorithm outperforms state-of-the-art methods in several challenging conditions.
\end{itemize}
Several complex algorithms have been proposed in the literature trying to handle all the challenges in background subtraction, for example the training of big Convolutional Neural Networks (CNN) models \cite{lim2018foreground,zeng2018multiscale,hu20183d}, or multicues techniques \cite{unzueta2011adaptive}. However it is not true, that one single model is able to effectively tackle all the challenges in background subtraction of unseen videos \cite{tezcan2020bsuv}. GraphBGS is a semi-supervised algorithm which lies just in the middle of unsupervised and supervised methods. Unlike other methods of the state-of-the-art, our algorithm: 1) works well for both static and moving camera sequences, 2) does not require a huge amount of data, and 3) is able to adapt to complex scenarios.

The rest of the paper is organized as follows. Section \ref{sec:algorithm} explains the notation, basic concepts, and the proposed algorithm. Section \ref{sec:experimental_framework} introduces the experimental framework of this paper. Finally, Sections \ref{sec:results_and_discussion} and \ref{sec:conclusions} present the results and conclusions, respectively.

\section{Background Subtraction via Reconstruction of Graph Signals}
\label{sec:algorithm}

Our proposed algorithm consists of several components including instance segmentation, background initialization, features extraction, graph construction, and semi-supervised learning as shown in the diagram in Fig. \ref{fig:pipeline}. We first compute the instance segmentation mask using Cascade Mask R-CNN \cite{cai2019cascade}. The instance segmentation algorithm aims to get the relevant object instances in our videos. We also compute a background model using a temporal median filter \cite{piccardi2004background}. Then, we use the background model as well as motion \cite{lucas1981iterative}, texture \cite{ojala2002multiresolution}, and intensity features to represent the nodes of a graph. Finally, we solve a semi-supervised learning problem inspired from the theory of GSP. The novices in GSP are referred to the review papers \cite{shuman2013emerging,ortega2018graph}, as well as the Appendix \ref{app:GSP}.


\subsection{Notation}
\label{sec:notation}

In this paper, upper case boldface letters such as $\mathbf{W}$ represent matrices, and lower case boldface letters such as $\mathbf{y}$ denote vectors. $\mathbf{Y}_{i,:}$ and $\mathbf{Y}_{:,j}$ are the $i$-th row vector and the $j$-th column vector of matrix $\mathbf{Y}$, respectively. $(\cdot)^{\mathsf{T}}$ represents transposition. $\sigma_{\text{max}}(\mathbf{A})$ and $\sigma_{\text{min}}(\mathbf{A})$ represent the maximum and minimum singular values of the matrix $\mathbf{A}$, respectively. Calligraphic letters such as $\mathcal{V}$ denote sets, and $\vert \mathcal{V} \vert$ represents its cardinality. $\diag(\mathbf{x})$ is a diagonal matrix with entries $x_1,x_2,\dots,x_n$. Finally, the $\ell_2$-norm of a vector is written as $\Vert \cdot \Vert_2$, and the inner product between two elements of a Euclidean space as $\langle \cdot,\cdot \rangle$.

\subsection{Preliminaries on Graph Signals}
\label{sec:background}

Let $G=(\mathcal{V},\mathcal{E})$ be an undirected, weighted, and connected graph; where $\mathcal{V}=\{1,\dots,N\}$ is the set of $N$ nodes and $\mathcal{E}=\{(i,j)\}$ is the set of edges. The $2$-tuple $(i,j)$ stands for an edge between two nodes $i$ and $j$. $\mathbf{W} \in \mathbb{R}^{N\times N}$ is the adjacency matrix of the graph such that $\mathbf{W}(i,j)=w_{ij} \in \mathbb{R}^+$ is the weight between nodes $i$ and $j$. As a consequence, $\mathbf{W}$ is symmetric for undirected graphs. A graph signal is a function $y:\mathcal{V}\to \mathbb{R}$ defined on the nodes of $G$, and can be represented as $\mathbf{y} \in \mathbb{R}^N$ where $\mathbf{y}(i)$ is the function evaluated on the $i$-th node. Furthermore, $\mathbf{D} \in \mathbb{R}^{N\times N}$ is a diagonal matrix called the degree matrix of the graph such that $\mathbf{D}(i,i)=\sum_{j=1}^N \mathbf{W}(i,j)~\forall~i = 1,2,\dots,N$. Likewise, the combinatorial Laplacian operator is a positive semi-definite matrix defined as $\mathbf{L} = \mathbf{D}-\mathbf{W}$, and it has eigenvalues $0=\lambda_1 \leq \lambda_2 \leq \dots \leq \lambda_N$ and corresponding eigenvectors $\{ \mathbf{u}_1,\mathbf{u}_2,\dots,\mathbf{u}_N\}$.

The graph Fourier basis of $G$ is defined by the spectral decomposition of $\mathbf{L} = \mathbf{U}\mathbf{\Lambda}\mathbf{U}^{\mathsf{T}}$, where $\mathbf{U}=[\mathbf{u}_1,\mathbf{u}_2,\dots,\mathbf{u}_N]$ and $\mathbf{\Lambda}=\diag(\lambda_1,\lambda_2,\dots,\lambda_N)$. Therefore, the Graph Fourier Transform (GFT) of the signal $\mathbf{y}$ is defined as $\mathbf{\hat{y}}=\mathbf{U}^{{\mathsf{T}}}\mathbf{y}$, and the inverse GFT is given by $\mathbf{y} = \mathbf{U}\mathbf{\hat{y}}$. Using this notion of frequency, Pesenson \cite{pesenson2008sampling} defined the space of all $\omega$\textit{-bandlimited} signals as $PW_{\omega}(G)=\spanmath(\mathbf{U}_{\rho}:\lambda_{\rho} \leq \omega)$, where $\mathbf{U}_{\rho}$ represents the first $\rho$ eigenvectors of $\mathbf{L}$, and $PW_{\omega}(G)$ is known as the Paley-Wiener space of $G$. As a consequence, a graph signal $\mathbf{y}$ has cutoff frequency $\omega$, and bandwidth $\rho$ if $\mathbf{y} \in PW_{\omega}(G)$.

Given the notion of bandlimitedness in terms of $PW_{\omega}(G)$, the next step is to find a bound for the minimum sampling rate allowing perfect recovery of $\mathbf{y} \in PW_{\omega}(G)$. Intuitively, a graph signal is smooth when $\mathbf{y} \in PW_{\omega}(G)$. For instance, suppose a temperature sensor network in any region of the world. The temperature of a certain amount of $\alpha$ nearby cities or localities, represented as the value of a graph signal in $\alpha$ strong connected nodes, should be similar. Then, probably one just needs some of these values of temperature to recover the value of the graph signal in the other nodes. One can also interpret this bound (for perfect recovery) as the minimum amount of labeled nodes required to have perfect classification in semi-supervised learning, given the prior assumption that the classes of the nodes lie in the Paley-Wiener space of the graph. In short, the answer to this question for the bound of minimum sampling rate is that one needs at least $\rho$ (bandwidth) sampled nodes to get perfect reconstruction.

Formally, the sampling rate of a graph signal is defined in terms of a subset of nodes $\mathcal{S} \subset \mathcal{V}$ with $\mathcal{S}=\{s_1,s_2,\dots,s_m\}$, where $m = \vert \mathcal{S}\vert \leq N$ is the number of sampled nodes. The sampled graph signal is defined as $\mathbf{y}(\mathcal{S})=\mathbf{My}$, where $\mathbf{M}$ is a binary decimation matrix whose entries are given by $\mathbf{M} = [\boldsymbol{\delta}_{s_1},\dots,\boldsymbol{\delta}_{s_m}]^{\mathsf{T}}$ and $\boldsymbol{\delta}_{v}$ is the $N-$dimensional Kronecker column vector (or the one-hot vector) centered at $v$. The reconstruction of the graph signal is achieved as $\tilde{\mathbf{y}} = \mathbf{\Phi My}$, where $\mathbf{\Phi} \in \mathbb{R}^{N \times m}$ is an interpolation operator. Perfect reconstruction happens when $\mathbf{\Phi M}$ is the identity matrix. This is not possible in general since $\rankmath(\mathbf{\Phi M}) \leq m \leq N$. However, perfect recovery from $\mathbf{y}(\mathcal{S})$ is possible if the sampling size $\vert \mathcal{S}\vert$ is greater than or equal to $\rho$, \ie $|\mathcal{S}| \geq \rho$ \cite{chen2015discrete}.

\begin{theorem}[Chen's theorem \cite{chen2015discrete}]
    \label{trm:perfect_reconstruction}
    Let $\mathbf{M}$ satisfy $\rankmath(\mathbf{MU}_{\rho})=\rho$. For all $\mathbf{y} \in PW_{\omega}(G)$, perfect recovery, \ie $\mathbf{y}=\mathbf{\mathbf{\Phi My}}$, is achieved by choosing:
    \begin{equation}
        \mathbf{\Phi} = \mathbf{U}_{\rho}\mathbf{V},
        \label{eqn:interpo_perfect_rec}
    \end{equation}
    with $\mathbf{V\Phi U}_{\rho}$ a $\rho\times \rho$ identity matrix.
    \\
    Proof: see \cite{chen2015discrete}.
\end{theorem}

Theorem \ref{trm:perfect_reconstruction} indicates that perfect recovery of graph signals is possible when $\mathbf{y}$ lie in the Paley-Wiener space of the graph, and the number of sampling nodes is at least $\rho$. This reconstruction is achieved by choosing the interpolation operator as in Eqn. (\ref{eqn:interpo_perfect_rec}). The computation of the Laplacian eigenvectors in Eqn. (\ref{eqn:interpo_perfect_rec}) is computationally prohibitive for large graphs. In this paper, the computation of the Laplacian eigenvectors is avoided. For further details, please see Section \ref{sec:reconstruction_algorithm}.

\subsection{Instance Segmentation}

The instance segmentation is obtained using a Cascade Mask Region Convolutional Neural Network (Cascade Mask R-CNN) \cite{cai2019cascade} with a ResNeSt-200 \cite{zhang2020resnest} plus Feature Pyramid Network (FPN) \cite{lin2017feature} as backbone. The Cascade Mask R-CNN is trained in the COCO 2017 dataset \cite{lin2014microsoft}. Each output of the instance segmentation network is a node in the proposed GraphBGS. The instance segmentation network gives a fundamental limitation to GraphBGS: if the Cascade Mask R-CNN does not segment a moving object, GraphBGS will not be able to detect it. Notice however that our pipeline in Fig. \ref{fig:pipeline} could work with other segmentation algorithms such as super-pixel \cite{achanta2012slic}, but the amount of nodes resulting from such an approach would lead to an extremely big graph. Currently, there are no techniques to handle such big graphs in the literature of GSP, but there are promising approaches \cite{parada2019blue,lu2019fast}.

\subsection{Background Initialization and Feature Extraction}
\label{sec:background_estimation_node_representation}

For the sake of simplicity, the computation of the background image is performed using the temporal median filter \cite{piccardi2004background}. However, most advanced methods in background initialization such as \cite{javed2017background,xu2019robust} could be used to improve the overall performance of the algorithm. The videos are processed in gray-scale in this paper.

The representation of the nodes are obtained using optical flow estimation, intensity, and texture features. Let $\mathbf{I}_v^t$ and $\mathbf{I}_v^{t-1}$ be the gray-scale cropped images corresponding to region of interest of the node $v \in \mathcal{V}$ in the current ($t$) and previous ($t-1$) frames, respectively. Let $\mathbf{B}_v$ and $\mathcal{P}_v$ be the cropped image of the background, and set of pixel indices corresponding to the node $v$, respectively. Furthermore, let $\mathbf{v}_x^t(\mathcal{P}_v)$ and $\mathbf{v}_y^t(\mathcal{P}_v)$ be the optical flow vectors of the current image with support in $\mathcal{P}_v$ for the horizontal and vertical direction, respectively, where the optical flow is computed with the Lucas-Kanade method \cite{lucas1981iterative}. $\mathbf{v}_x^t(\mathcal{P}_v)$ and $\mathbf{v}_y^t(\mathcal{P}_v)$ are used to computed histograms and some descriptive statistics (the minimum, maximum, mean, standard deviation, mean absolute deviation, and range). Furthermore, intensity histograms are computed in $\mathbf{I}_v^t(\mathcal{P}_v)$, $\mathbf{I}_v^{t-1}(\mathcal{P}_v)$, $\mathbf{B}_v(\mathcal{P}_v)$ and $\vert \mathbf{I}_v^t(\mathcal{P}_v) - \mathbf{B}_v(\mathcal{P}_v) \vert$. And finally, the texture representation is computed in $\mathbf{I}_v^t$, $\mathbf{I}_v^{t-1}$, $\mathbf{B}_v$ and $\vert \mathbf{I}_v^t - \mathbf{B}_v \vert$, using local binary patterns \cite{ojala2002multiresolution}. The representation of the node is obtained concatenating all the previous features, \ie optical flow, intensity, and texture features. Each instance $v$ is represented by a $853$-dimensional vector $\mathbf{x}_v$.

\subsection{Graph Construction}

Let $\mathbf{X} \in \mathbb{R}^{N\times M}$ be the matrix of $N$ instances, where $\mathbf{X}=[\mathbf{x}_1,\mathbf{x}_2,\dots,\mathbf{x}_N]^{\mathsf{T}}$. Firstly, a k-nearest neighbors algorithm with $\text{k}=30$ is used to connect the nodes in the graph. Afterwards, vertices are connected to get an undirected graph. The weight between two connected nodes $i,j$ is given such that $w_{ij} = \exp{{-\frac{d(i,j)^2}{\sigma^2}}}$, where $d(i,j)=\Vert \mathbf{x}_i - \mathbf{x}_j \Vert$, and $\sigma^2$ is the standard deviation of the Gaussian function computed as $\sigma = \frac{1}{\vert \mathcal{E} \vert + N}\sum_{(i,j) \in \mathcal{E}} d(i,j)$. Notice that the ground truth information is not used for the construction of the graph, but the construction of the graph signals.

\subsection{Graph Signal}

The graph signal in semi-supervised learning is a matrix $\mathbf{Y} \in \mathbb{R}^{N\times Q}$, where $Q$ is the number of classes of the problem, and $\mathbf{Y}_{i,:}~\forall~ i \in \mathcal{V}$ is the representation of certain class in the node $i$. In this paper, each row of $\mathbf{Y}$ encodes the class of each node with the so called one-hot vector, in this case $Q$ is equal to $2$ corresponding to the classes: background $[1,0]$, and foreground $[0,1]$. The decision of whether a node is background or foreground is based on a comparison between the ground-truth and the instances of the videos. The proposed method uses the metrics \textit{intersection over union} and \textit{intersection over node} to determine the foreground and background nodes.

Consider the following definitions: let $\mathcal{GT}^t$ be the set of indices of the foreground pixels in the ground-truth in the current image $t$; let $\mathcal{P}_v$ be the set of indices of the pixels of each output of the Mask R-CNN with $v \in \mathcal{V}$, where the subset of nodes in $\mathcal{V}$ associated with the current image $t$ is given by the masks that lie in such a frame; let $\mathcal{F}^t=\{1,...,\gamma\}$ be the set of isolated regions of the ground-truth in the current image $t$; let $\mathcal{GT}_f^t$ be the set of foreground indices of each isolated region of the ground-truth in the frame $t$ with $f \in \mathcal{F}^t$, \ie $\mathcal{GT}_f^t \subseteq \mathcal{GT}^t$. Using the definitions above, the intersection over node is such that $\xi = \vert \mathcal{I} \vert/\vert \mathcal{P}_v \vert$, where $\mathcal{I}=\mathcal{GT}^t \cap \mathcal{P}_v$; and the intersection over union is such that $\mathbf{u}(f) = \vert \mathcal{I}_f \vert/\vert \mathcal{U}_f \vert~\forall~f \in \mathcal{F}^t$, where $\mathcal{U}_f=\mathcal{GT}_f^t \cup \mathcal{P}_v$, and $\mathcal{I}_f=\mathcal{GT}_f^t \cap \mathcal{P}_v$, \ie each node $v \in \mathcal{V}$ has $\gamma$ values of intersection over union.

One can say that, either if $\mathcal{F} \in \emptyset$, or $\mu=0$, or $\xi=0$, their corresponding nodes $v$ are background, \ie $\mathbf{Y}_{v,:} = [1,0]$. On the contrary, when $\mu > 0.25$, or when $\xi > 0.45$ and $\mu > 0.05$, or when $\xi > 0.9$ and $\mu > 0.02$, their corresponding nodes $v$ are foreground, \ie $\mathbf{Y}_{v,:} = [0,1]$. Otherwise, the corresponding node $v$ is classified as background. The parameters before were found empirically, but these numbers can be learned with a regression method for example. An expert would directly classify if certain node is background or foreground in a new database for background subtraction with semi-supervised learning. As a consequence, the tuning of the parameters for the construction of the graph signal is not a critical step in the pipeline of GraphBGS.

\subsection{Semi-supervised Learning Algorithm}
\label{sec:reconstruction_algorithm}

The semi-supervised learning algorithm of this paper is based on the variational splines of Pesenson \cite{pesenson2009variational}.

\begin{definition}
    \label{dfn:Hilbert_space}
    Let $\mathbf{g}^*(v)$ be the complex conjugate of $\mathbf{g}(v)$. The space $L_2(G)$ is the Hilbert space of all complex-valued functions $\mathbf{f}:\mathcal{V} \rightarrow \mathbb{C}$ in the graph $G$ with $\Vert \mathbf{f} \Vert = \Vert \mathbf{f} \Vert_0 = \left( \sum_{v\in \mathcal{V}} \vert \mathbf{f}(v) \vert^2 \mathbf{D}(v,v) \right)^{1/2}$ and the following inner product:
    \begin{equation}
        \langle \mathbf{f},\mathbf{g} \rangle = \sum_{v\in \mathcal{V}} \mathbf{f}(v)\mathbf{g}^*(v) \mathbf{D}(v,v).
    \end{equation}
\end{definition}

\begin{definition}
    \label{dfn:sobolev_norm}
    For a fixed $\epsilon \geq 0$ the Sobolev norm is introduced by the following formula:
    \begin{equation}
        \Vert \mathbf{f} \Vert_{\beta,\epsilon} = \Vert (\mathbf{L}+\epsilon \mathbf{I})^{\beta/2} \mathbf{f} \Vert, \beta \in \mathbb{R}.
        \label{eqn:sobolev_norm}
    \end{equation}
\end{definition}

Let $\mathbf{y}_q=\mathbf{Y}_{:,q}$ be the graph signal associated with the $q$-th class. Given a subset of vertices $\mathcal{S} \subset \mathcal{V}$, a sampled vector of labels $\mathbf{y}_q(\mathcal{S})=\mathbf{My}_q$, a positive $\beta>0$, and a non-negative $\epsilon \geq 0$, the variational problem for semi-supervised learning is stated as follow: find a vector $\mathbf{z}_q$ from the space $L_2(G)$ with the following properties: $\mathbf{z}_q(\mathcal{S})=\mathbf{Mz}_q=\mathbf{y}_q(\mathcal{S})$, and $\mathbf{z}_q$ minimizes functional $\mathbf{z}_q \rightarrow \Vert (\mathbf{L}+\epsilon \mathbf{I})^{\beta/2}\mathbf{z}_q \Vert$. In other words, the variational problem is trying to solve the following optimization problem:
\begin{gather}
    \nonumber
    \argmin_{\mathbf{z}_q} \Vert \mathbf{z}_q \Vert_{\beta,\epsilon} \quad \text{s.t.} \quad \mathbf{Mz}_q = \mathbf{y}_q(\mathcal{S}) \rightarrow \\
    \label{eqn:variational_problem_rewritten}
    \argmin_{\mathbf{z}_q} \mathbf{z}_q^{\mathsf{T}} (\mathbf{L}+\epsilon \mathbf{I})^{\beta} \mathbf{z}_q \quad \text{s.t.} \quad \mathbf{Mz}_q - \mathbf{y}_q(\mathcal{S}) = 0.
\end{gather}
In the case of background subtraction, Eqn. (\ref{eqn:variational_problem_rewritten}) is solved for $q=1,2$. The variational problem is a convex optimization problem because the term $\mathbf{z}_q^{\mathsf{T}} (\epsilon \mathbf{I}+\mathbf{L})^{\beta} \mathbf{z}_q$ is a quadratic convex function, and the term $\mathbf{Mz}_q - \mathbf{y}_q(\mathcal{S})$ is affine in $\mathbf{z}_q$. Even though $\detmath(\mathbf{L})=0$ for undirected and connected graphs, the term $(\mathbf{L}+\epsilon \mathbf{I})$ is always invertible for $\epsilon>0$. Intuitively, the value $\epsilon$ is related to the stability of the inverse of $(\mathbf{L}+\epsilon \mathbf{I})$.

\begin{theorem}
    \label{trm:perturbation_matrix}
    Given an undirected and connected graph $G$ with combinatorial Laplacian matrix $\mathbf{L}$ such that $\text{rank}(\mathbf{L}) = N-1$, and let $\mathbf{\Psi} \in \mathbb{R}^{N\times N}$ be a perturbation matrix. The summation $\mathbf{L} + \mathbf{\Psi}$ has a lower and upper bound in the condition number in the $\ell_2$-norm such that:
    \begin{equation}
        \frac{\sigma_{\text{max}}(\mathbf{L}+\mathbf{\Psi})}{\sigma_{\text{max}}(\mathbf{\Psi})} \leq \kappa(\mathbf{L}+\mathbf{\Psi}) \leq \frac{\sigma_{\text{max}}(\mathbf{L})+\sigma_{\text{max}}(\mathbf{\Psi})}{\sigma_{min}(\mathbf{L}+\mathbf{\Psi})},
        \label{eqn:condition_number}
    \end{equation}
    where $\kappa(\mathbf{L}+\mathbf{\Psi})$ is the condition number of $\mathbf{L}+\mathbf{\Psi}$.\\
    Proof: see Appendix \ref{app:perturbation_matrix}.
\end{theorem}

Theorem \ref{trm:perturbation_matrix} provides a lower and upper bound in the condition number\footnote{The condition number $\kappa(\mathbf{A})$ associated with the square matrix $\mathbf{A}$ is a measure of how well or ill conditioned is the inversion of $\mathbf{A}$.} of $\mathbf{L}+\mathbf{\Psi}$, where the lower bound can be achieved by computing the spectral decomposition of $\mathbf{L}$. The addition of a perturbation matrix to the Laplacian matrix is implicitly changing the eigenvalues of $\mathbf{L}$, however Theorem \ref{trm:perturbation_matrix} does not state how these eigenvalues change. In matrix theory, Weyl's inequality \cite{horn2012matrix} is a theorem about how the eigenvalues of a perturbed matrix change.

\begin{theorem}[Weyl's Theorem \cite{horn2012matrix}]
    \label{trm:weyl_theorem}
    Let $\mathbf{L}$ and $\mathbf{\Psi}$ be Hermitian matrices with set of eigenvalues $\{\lambda_1,\lambda_2,\dots,\lambda_N\}$ and $\{\psi_{1},\psi_{2},\dots,\psi_{N}\}$, respectively. The matrix $\mathbf{L} + \mathbf{\Psi}$ has a set of eigenvalues $\{\nu_{1},\nu_{2},\dots,\nu_{N}\}$ where the following inequalities hold (for $i=1,2,\dots,N$) $\lambda_i+\psi_1 \leq \nu_i \leq \lambda_i + \psi_N$.\\
    Proof: see \cite{horn2012matrix}.
\end{theorem}
In Theorem \ref{trm:weyl_theorem}, if $\mathbf{\Psi} \succ \boldsymbol{0}$, \ie $\psi_i > 0~\forall~1\leq i \leq N$, then this implies that $\nu_i > \lambda_i~\forall~i=1,2,\dots,N$. Weyl's Theorem indicates how the eigenvalues of $\mathbf{L}$ change after adding a perturbation matrix, and it gives insights about the structure of $\mathbf{\Psi}$. It is desirable to have $\detmath(\mathbf{L} + \mathbf{\Psi}) \neq 0$, then $\mathbf{\Psi}$ should be positive definite. Assuming $\mathbf{\Psi}=\epsilon\mathbf{I}$, where $\epsilon \in \mathbb{R}^+$ and $\mathbf{I}$ is the identity matrix, we have $\sigma_{\text{max}}(\epsilon\mathbf{I})=\epsilon$. Furthermore, $\sigma_{\text{min}}(\mathbf{L}+\epsilon\mathbf{I})$ is strictly greater than zero according to Theorem \ref{trm:weyl_theorem}, \ie $\nu_1 > \lambda_1$. Since $\sigma_{\text{min}}(\mathbf{L}+\epsilon\mathbf{I})>0$, then the upper bound in Eqn. (\ref{eqn:condition_number}) for $\mathbf{\Psi} = \epsilon\mathbf{I}$ is:
\begin{equation}
    \kappa(\mathbf{L}+\epsilon\mathbf{I}) \leq \frac{\sigma_{\text{max}}(\mathbf{L})+\epsilon}{\sigma_{min}(\mathbf{L}+\epsilon\mathbf{I})} < \infty.
\end{equation}
The term $(\mathbf{L}+\epsilon \mathbf{I})$ is precisely the first term in the Sobolev norm in Eqn. (\ref{eqn:sobolev_norm}), and the invertible term in the variational problem in Eqn. (\ref{eqn:variational_problem_rewritten}). $\epsilon$ is fundamentally related to how well conditioned is the variational problem. Larger values of $\epsilon$ leads to better condition numbers, but also to larger changes of the eigenvalues of the Laplacian matrix. Since $(\mathbf{L}+\epsilon \mathbf{I})$ is always invertible for $\epsilon>0$, the optimization problem in Eqn. (\ref{eqn:variational_problem_rewritten}) has a closed-form solution given by:
\begin{equation}
    \tilde{\mathbf{Z}} = ((\mathbf{L}+\epsilon \mathbf{I})^{-1})^{\beta}\mathbf{M}^{\mathsf{T}}(\mathbf{M}((\mathbf{L}+\epsilon \mathbf{I})^{-1})^{\beta}\mathbf{M}^{\mathsf{T}})^{-1}\mathbf{Y}(\mathcal{S}),
    \label{eqn:closed_form_solution}
\end{equation}
where $\tilde{\mathbf{Z}}=[\mathbf{z}_1,\mathbf{z}_2]$, and $\mathbf{Y}(\mathcal{S})$ is the sub-matrix of $\mathbf{Y}$ with rows indexed by $\mathcal{S}$. Equation (\ref{eqn:closed_form_solution}) can be used when the amount of nodes $N$ is relatively small.

Experimentally, $\epsilon$ is set to $0.2$ and $\beta=1$ in this paper. The final classification $\tilde{\mathbf{Y}}_{i,:}$ of node $i$ is computed as the one-hot vector centered in the maximum value of the corresponding reconstructed vector $\tilde{\mathbf{Z}}_{i,:}$. In this paper, the optimization problem in Eqn. (\ref{eqn:variational_problem_rewritten}) is solved using the graph signal processing toolbox \cite{perraudin2014gspbox}.

\section{Experimental Framework}
\label{sec:experimental_framework}

This section introduces the databases used in this paper, the evaluation metrics, the experiments, and the implementation details of GraphBGS.

\subsection{Databases}

The Change Detection (CDNet2014)\footnote{\url{http://www.changedetection.net}} \cite{wang2014cdnet} and the UCSD background subtraction \cite{mahadevan2009spatiotemporal} databases were used in this paper. Even though there are other databases for background subtraction \cite{li2016weighted,camplani2017rgb}, CDNet2014 is still the reference in the community. CDNet2014 contains $11$ challenges: bad weather, low frame rate, night videos, PTZ, turbulence, baseline, dynamic background, camera jitter, intermittent object motion, shadow, and thermal. Each challenge contains from $4$ up to $6$ videos, and each video has from $600$ up to $7999$ frames. Every video contains a certain amount of ground-truth frames, in which the ground truth shows the foreground and background. On the other hand, UCSD contains $18$ videos mainly composed of moving camera sequences, with $30$ up to $246$ frames each sequence. Each video of UCSD is partially or fully annotated with pixel-level ground-truth images of foreground and background.


\subsection{Evaluation Metrics}

The F-measure is the metric used to compare GraphBGS with the state-of-the-art methods, and it is defined as follows:
\begin{gather}
    \label{eqn:f-measure}
    \text{F-measure} = 2\frac{\text{Precision}\times \text{Recall}}{\text{Precision}+\text{Recall}},\\
    \text{Recall} = \frac{\text{TP}}{\text{TP}+\text{FN}}, \text{ Precision} = \frac{\text{TP}}{\text{TP}+\text{FP}},
\end{gather}
where TP, FP, and FN are the number of True Positives, False Positives, and False Negatives pixels, respectively.

\begin{table*}[htb]
\centering
\caption{Some visual results on CDNet2014 dataset compared with state-of-the-art methods, from left to right: original images, ground-truth images, SuBSENSE \cite{st2014subsense}, PAWCS \cite{st2015self}, IUTIS-5 \cite{bianco2017combination}, BSUV-Net \cite{tezcan2020bsuv}, and the proposed GraphBGS.}
\begin{tabular}{|lccccccc|}
\hline 
\footnotesize \textbf{Categories} & \scriptsize Original & \scriptsize Ground Truth  & \scriptsize SuBSENSE & \scriptsize PAWCS & \scriptsize IUTIS-5 & \scriptsize BSUV-Net & \scriptsize GraphBGS ({\color[HTML]{FE0000}ours})\\
\hline 
\hline 
\begin{tabular}[l]{@{}l@{}} \small Bad Weather\\ \small Snow Fall\\ \small  in002776\end{tabular} & \parbox[c][13.5mm][c]{18mm}{\includegraphics[height=12.6mm]{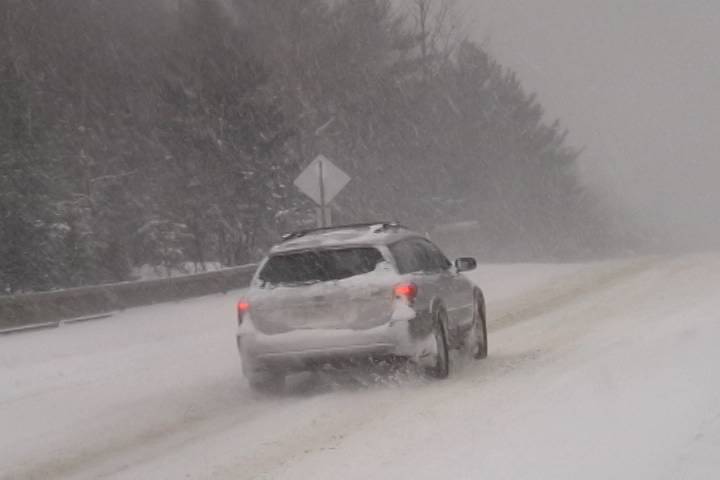}} & \parbox[c][13.5mm][c]{18mm}{\includegraphics[height=12.6mm]{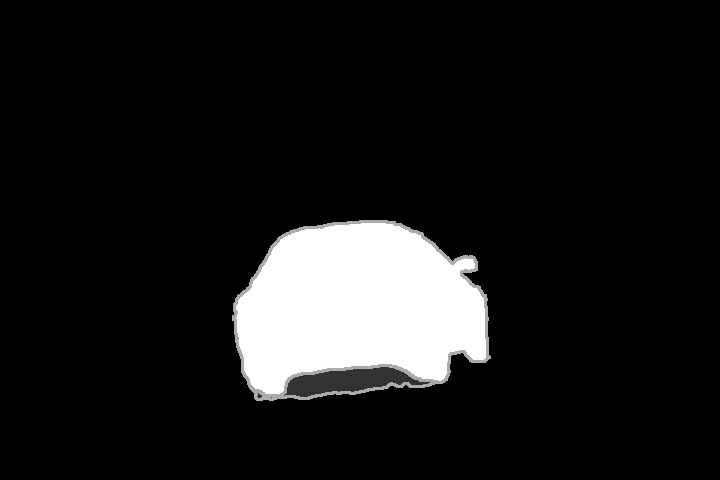}} & \parbox[c][13.5mm][c]{18mm}{\includegraphics[height=12.6mm]{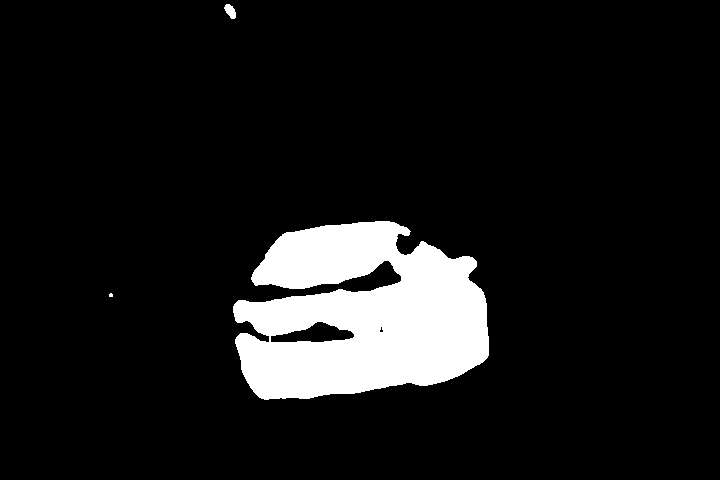}} & \parbox[c][13.5mm][c]{18mm}{\includegraphics[height=12.6mm]{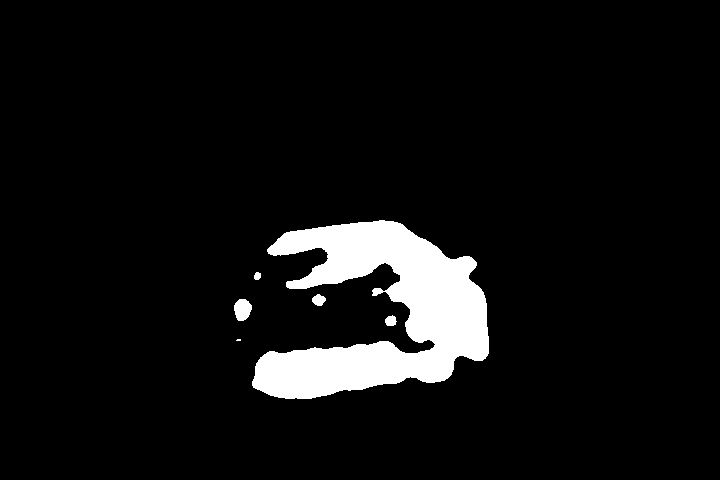}} & \parbox[c][13.5mm][c]{18mm}{\includegraphics[height=12.6mm]{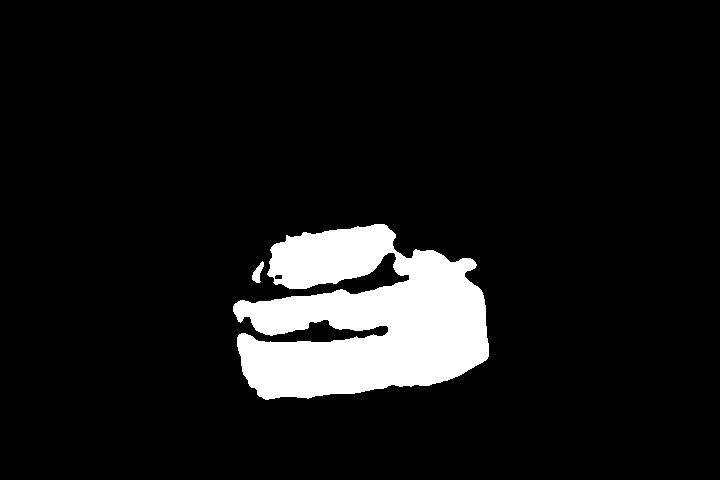}} & \parbox[c][13.5mm][c]{18mm}{\includegraphics[height=12.6mm]{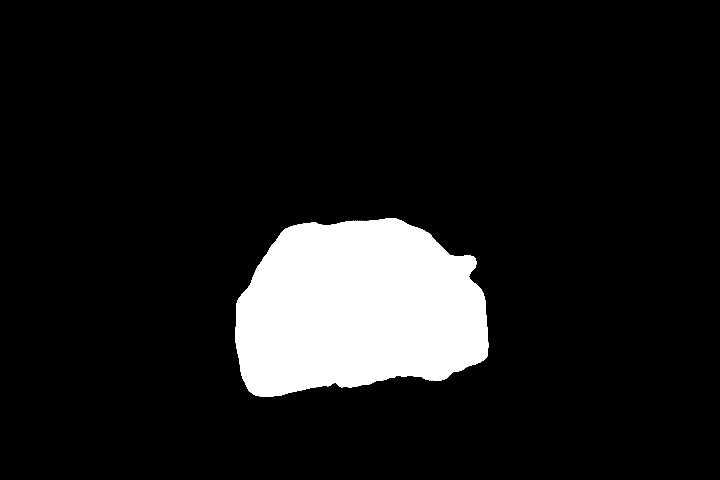}} & \parbox[c][13.5mm][c]{18mm}{\includegraphics[height=12.6mm]{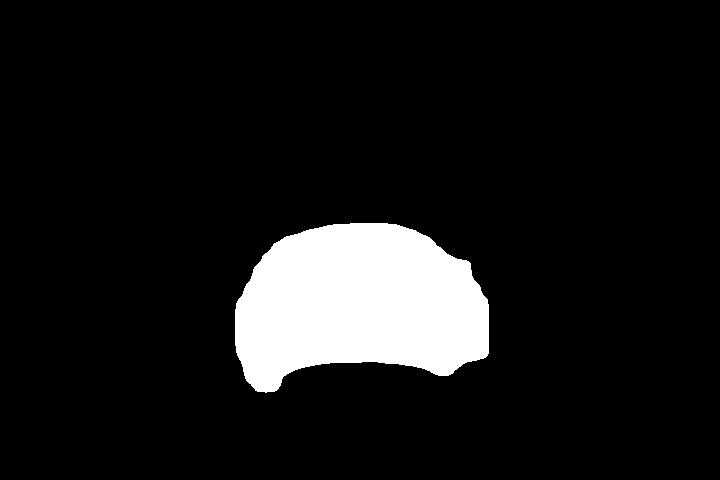}} \\ 
\hline  
\begin{tabular}[l]{@{}l@{}} \small Baseline\\ \small PETS2006\\ \small in000986\end{tabular} & \parbox[c][16mm][c]{18mm}{\includegraphics[height=15.3mm]{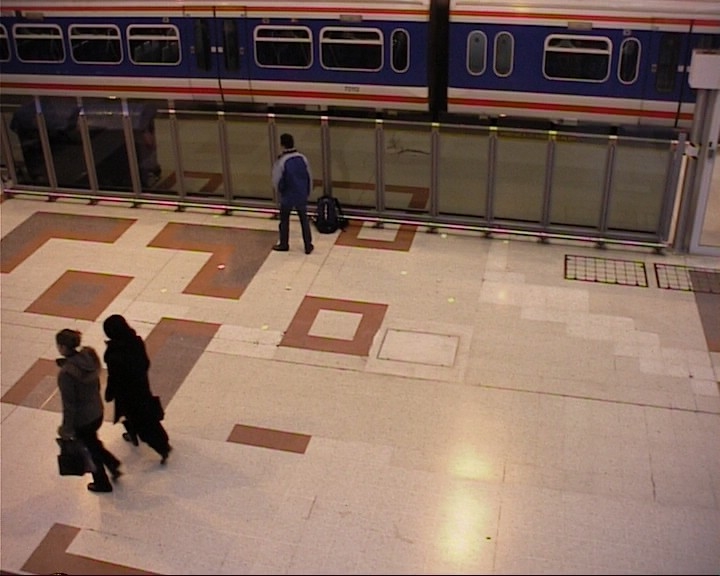}} & \parbox[c][16mm][c]{18mm}{\includegraphics[height=15.3mm]{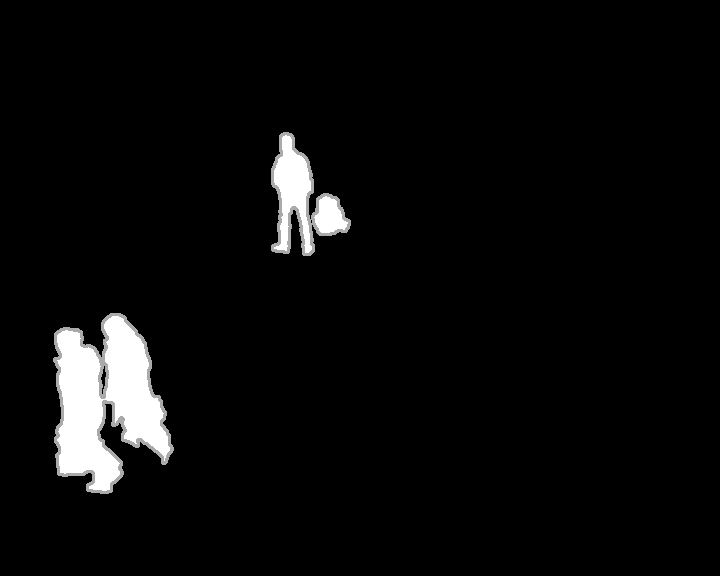}} & \parbox[c][16mm][c]{18mm}{\includegraphics[height=15.3mm]{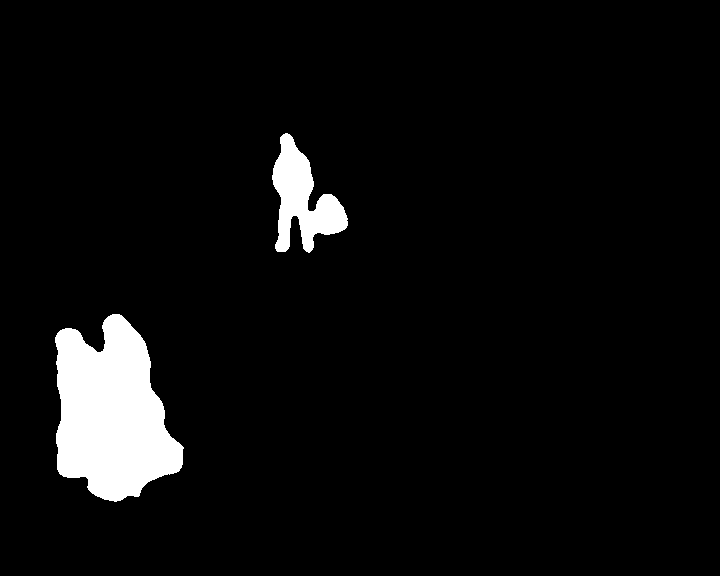}} & \parbox[c][16mm][c]{18mm}{\includegraphics[height=15.3mm]{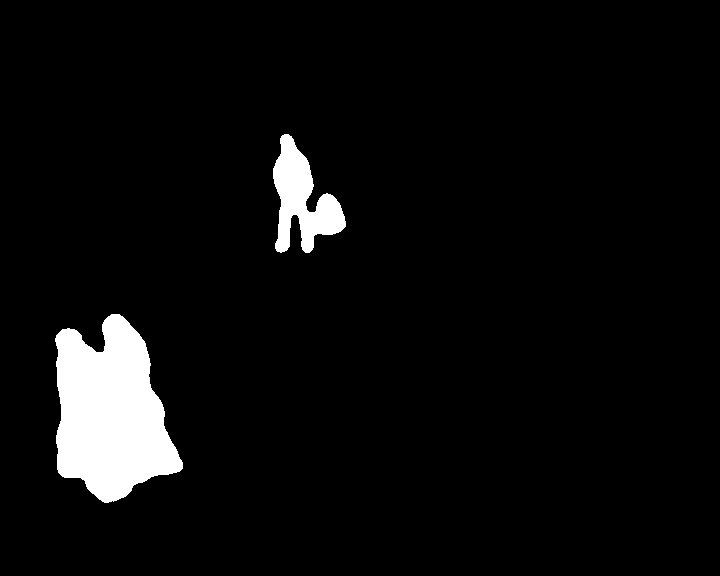}} & \parbox[c][16mm][c]{18mm}{\includegraphics[height=15.3mm]{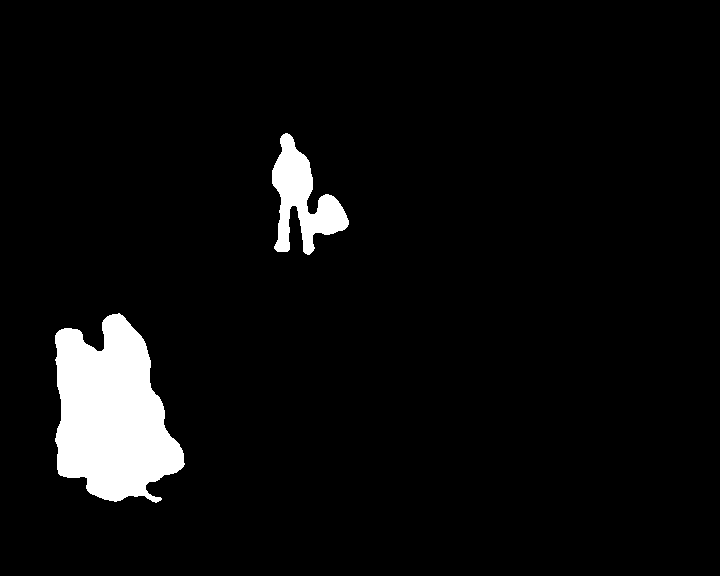}} & \parbox[c][16mm][c]{18mm}{\includegraphics[height=15.3mm]{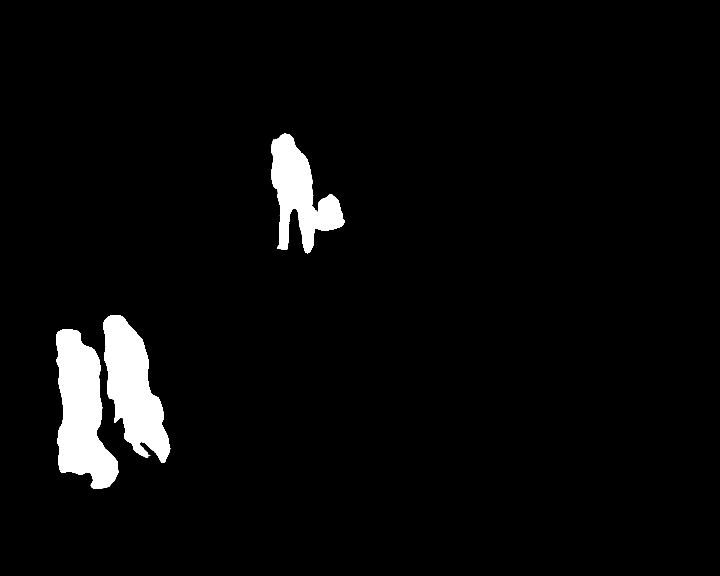}} & \parbox[c][16mm][c]{18mm}{\includegraphics[height=15.3mm]{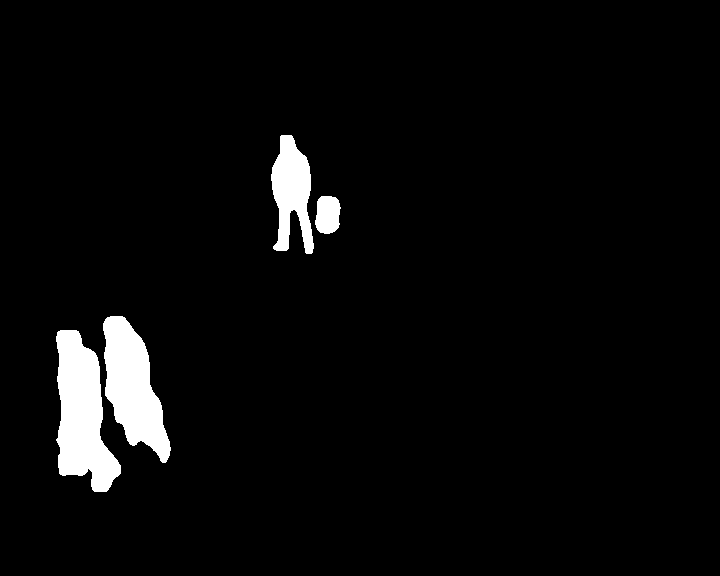}} \\ 
\hline  
\begin{tabular}[l]{@{}l@{}} \small Camera Jitter\\ \small Badminton\\ \small in000980\end{tabular} & \parbox[c][13.5mm][c]{18mm}{\includegraphics[height=12.8mm]{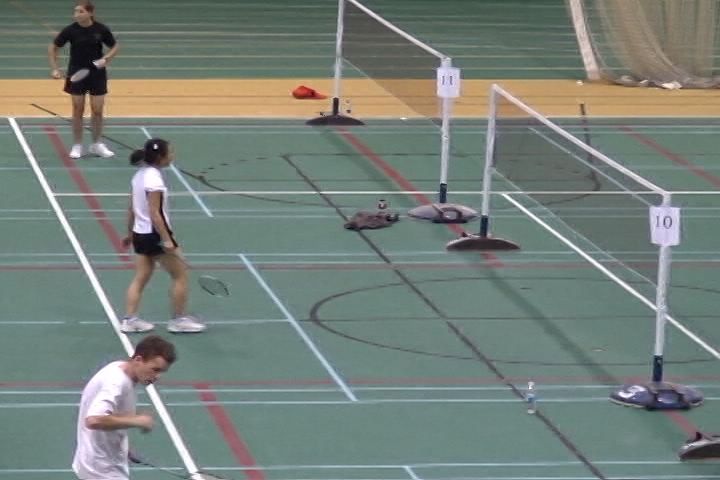}} & \parbox[c][13.5mm][c]{18mm}{\includegraphics[height=12.8mm]{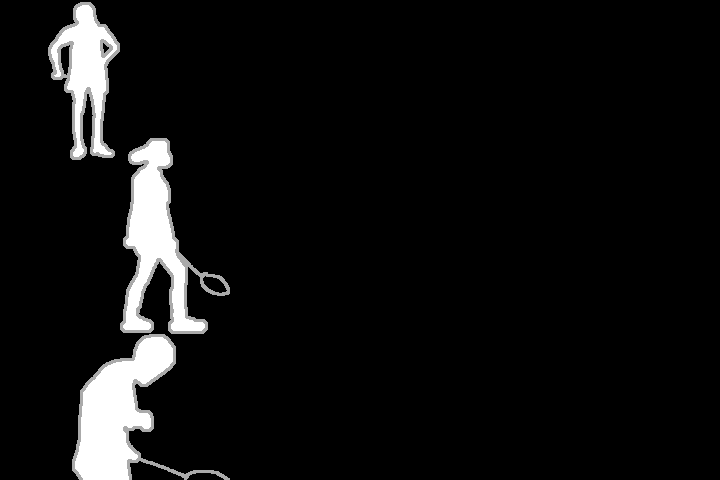}} & \parbox[c][13.5mm][c]{18mm}{\includegraphics[height=12.8mm]{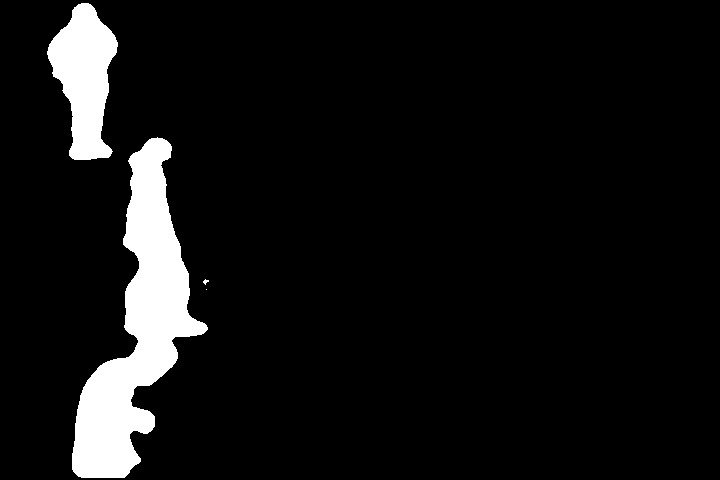}} & \parbox[c][13.5mm][c]{18mm}{\includegraphics[height=12.8mm]{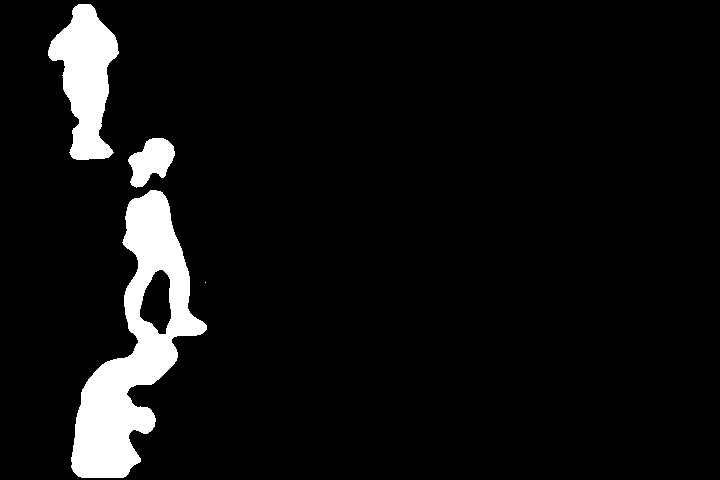}} & \parbox[c][13.5mm][c]{18mm}{\includegraphics[height=12.8mm]{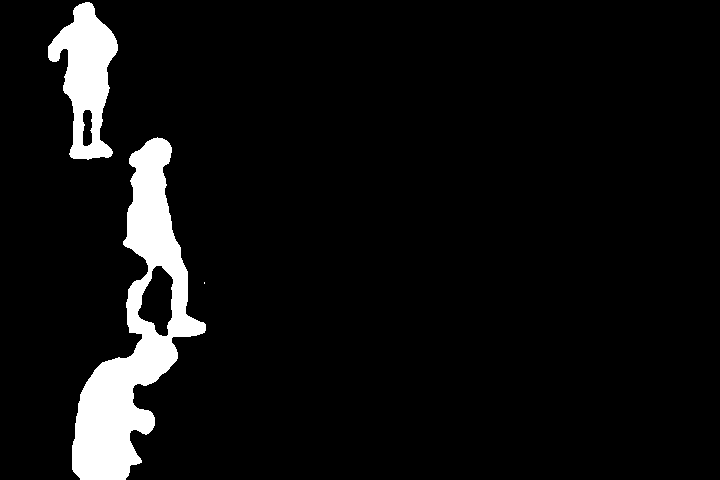}} & \parbox[c][13.5mm][c]{18mm}{\includegraphics[height=12.8mm]{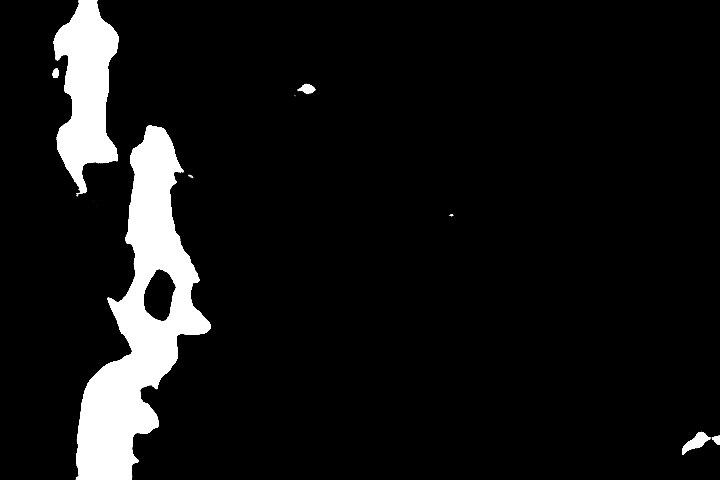}} & \parbox[c][13.5mm][c]{18mm}{\includegraphics[height=12.8mm]{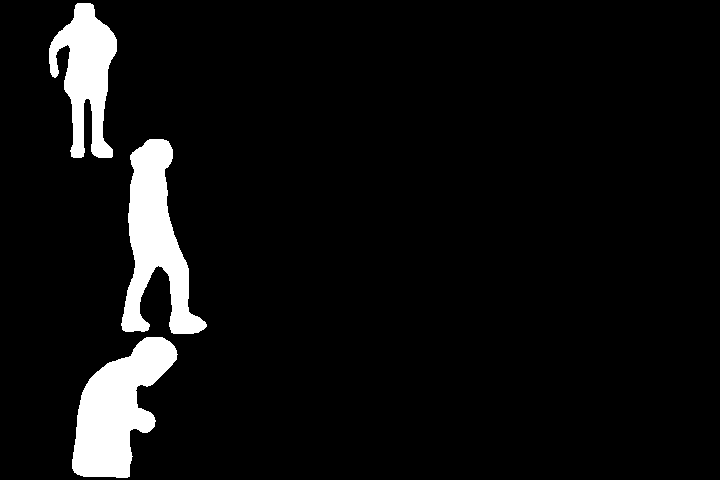}} \\ 
\hline  
\begin{tabular}[l]{@{}l@{}} \small Dynamic-B\\ \small Fall\\ \small in002795\end{tabular} & \parbox[c][13.5mm][c]{18mm}{\includegraphics[height=12.8mm]{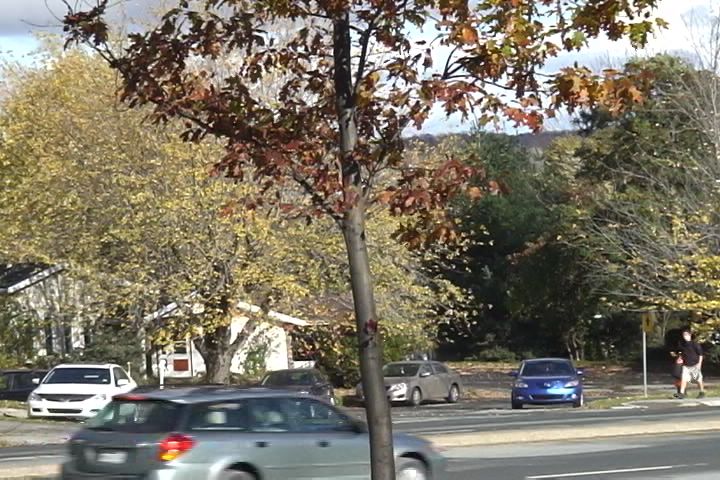}} & \parbox[c][13.5mm][c]{18mm}{\includegraphics[height=12.8mm]{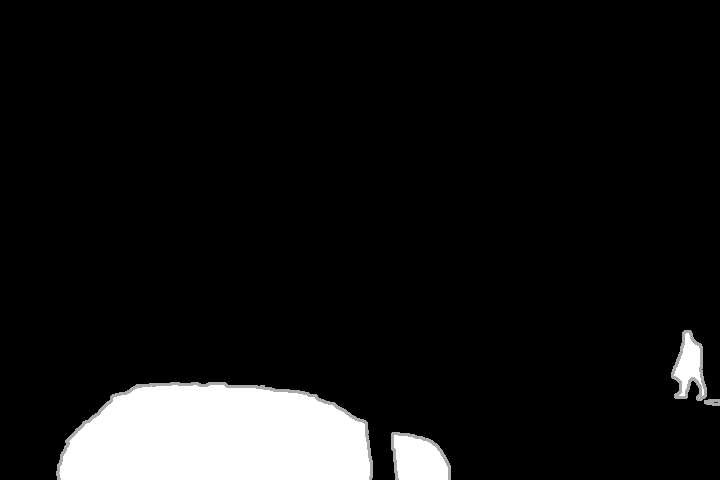}} & \parbox[c][13.5mm][c]{18mm}{\includegraphics[height=12.8mm]{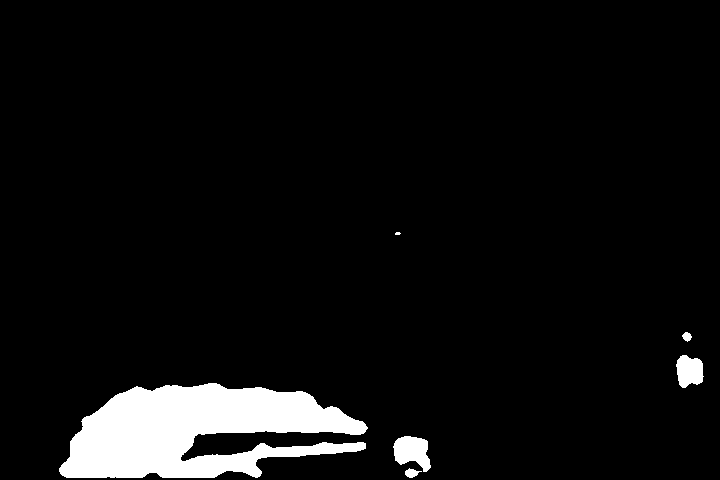}} & \parbox[c][13.5mm][c]{18mm}{\includegraphics[height=12.8mm]{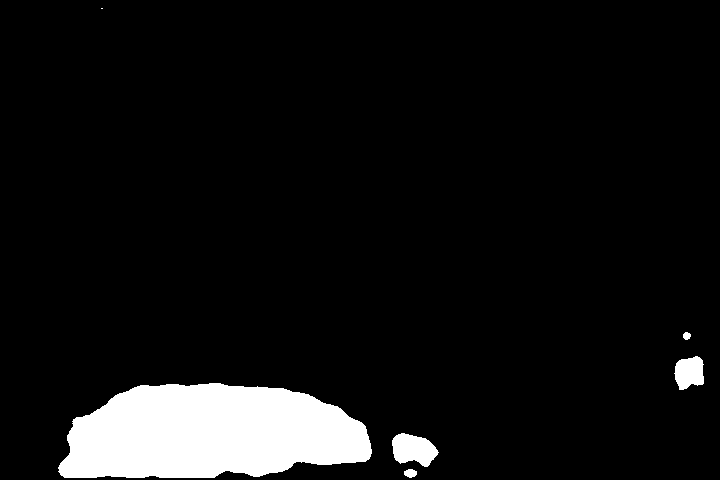}} & \parbox[c][13.5mm][c]{18mm}{\includegraphics[height=12.8mm]{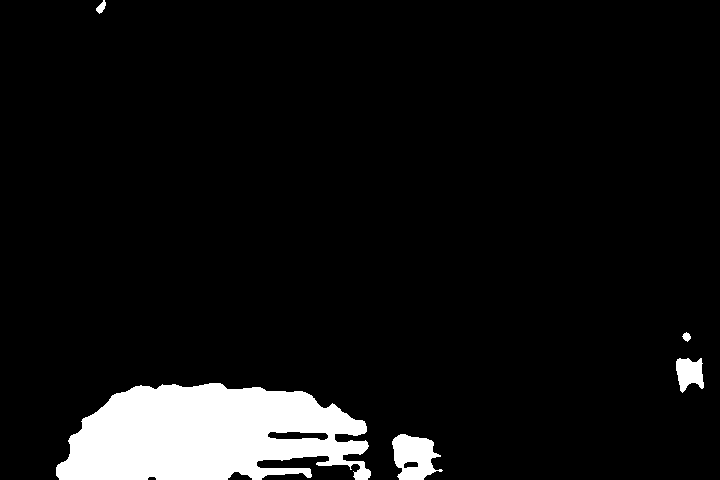}} & \parbox[c][13.5mm][c]{18mm}{\includegraphics[height=12.8mm]{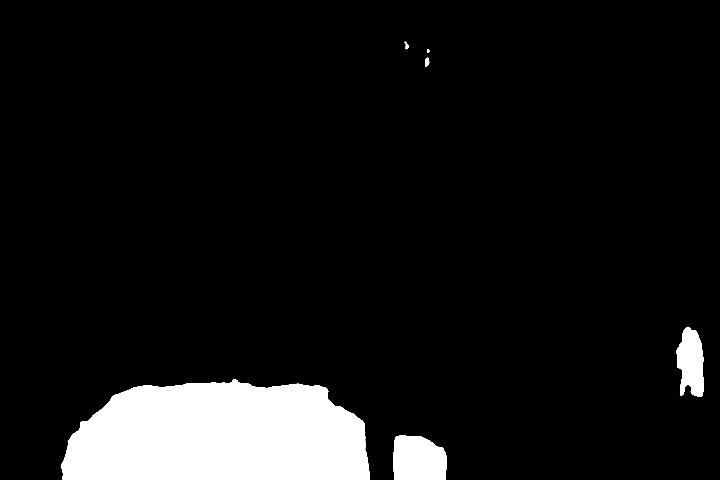}} & \parbox[c][13.5mm][c]{18mm}{\includegraphics[height=12.8mm]{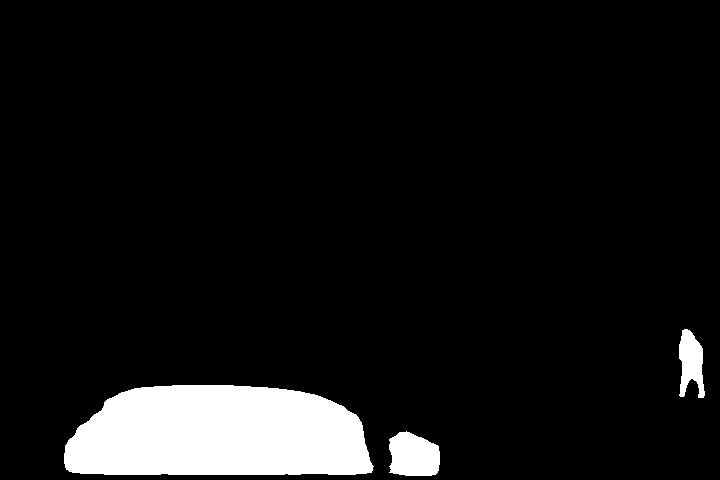}} \\ 
\hline  
\begin{tabular}[l]{@{}l@{}} \small I-O Motion\\ \small Sofa\\ \small in000651\end{tabular} & \parbox[c][15mm][c]{18mm}{\includegraphics[height=14.3mm]{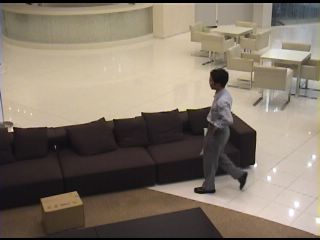}} & \parbox[c][15mm][c]{18mm}{\includegraphics[height=14.3mm]{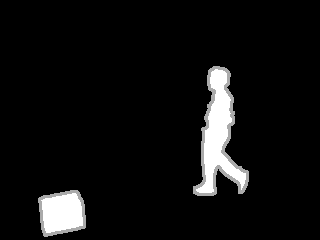}} & \parbox[c][15mm][c]{18mm}{\includegraphics[height=14.3mm]{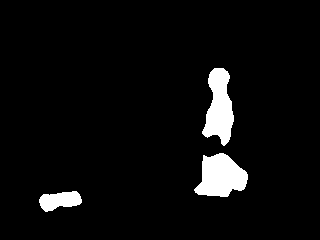}} & \parbox[c][15mm][c]{18mm}{\includegraphics[height=14.3mm]{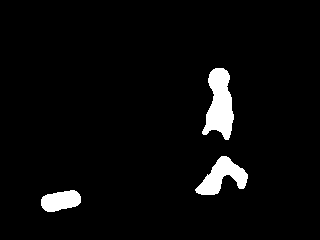}} & \parbox[c][15mm][c]{18mm}{\includegraphics[height=14.3mm]{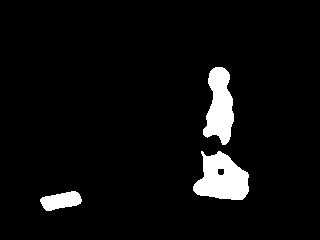}} & \parbox[c][15mm][c]{18mm}{\includegraphics[height=14.3mm]{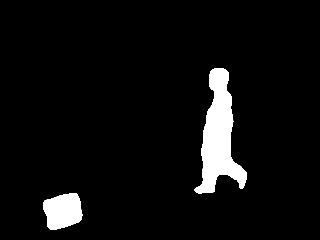}} & \parbox[c][15mm][c]{18mm}{\includegraphics[height=14.3mm]{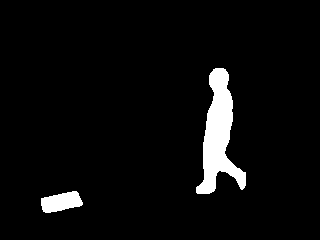}} \\ 
\hline  
\begin{tabular}[l]{@{}l@{}} \small PTZ\\ \small Intermittent-P\\ \small in001873\end{tabular} & \parbox[c][13.5mm][c]{18mm}{\includegraphics[height=12.6mm]{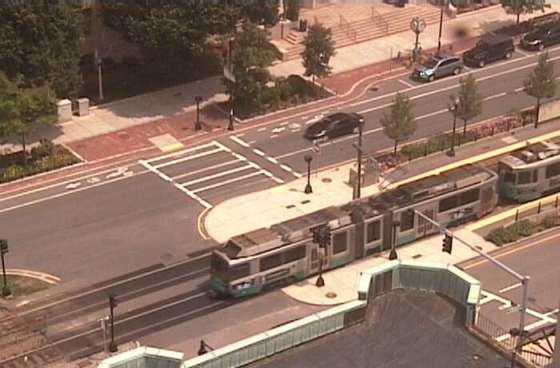}} & \parbox[c][13.5mm][c]{18mm}{\includegraphics[height=12.6mm]{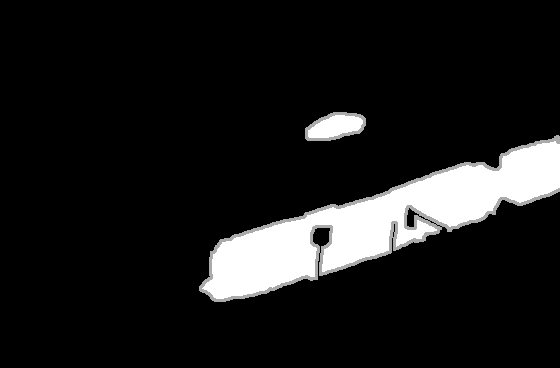}} & \parbox[c][13.5mm][c]{18mm}{\includegraphics[height=12.6mm]{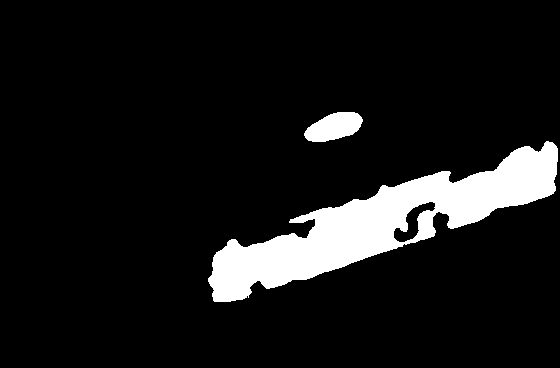}} & \parbox[c][13.5mm][c]{18mm}{\includegraphics[height=12.6mm]{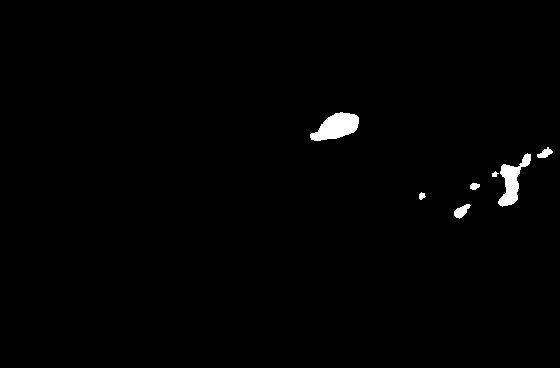}} & \parbox[c][13.5mm][c]{18mm}{\includegraphics[height=12.6mm]{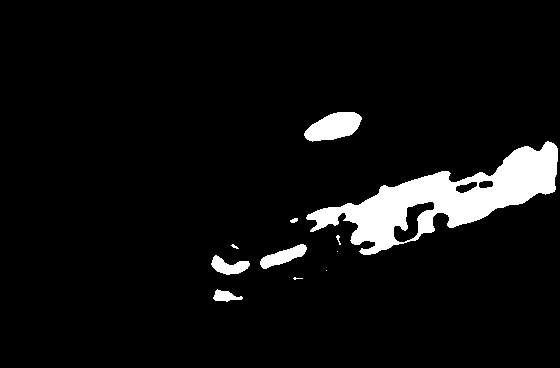}} & \parbox[c][13.5mm][c]{18mm}{\includegraphics[height=12.6mm]{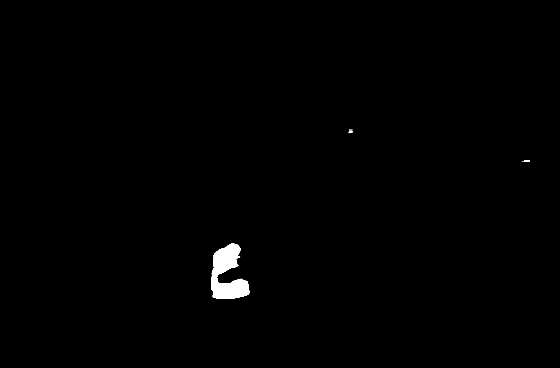}} & \parbox[c][13.5mm][c]{18mm}{\includegraphics[height=12.6mm]{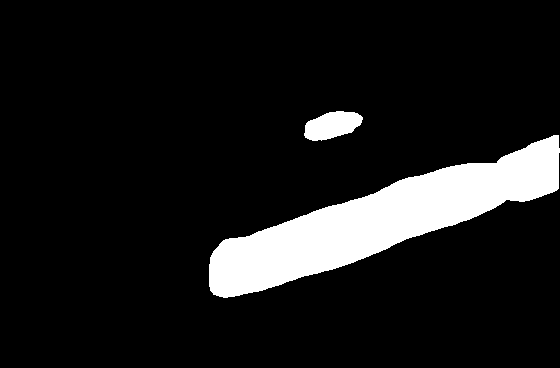}} \\ 
\hline  
\end{tabular}
\label{tbl:visual_results}
\end{table*}

\subsection{Experiments}

There are two experiments in this paper, the first one is related to the CDNet2014, and the second one to the UCSD dataset. In the first experiment, a graph is constructed for the whole CDNet2014, resulting in a graph of $310289$ nodes. In addition, for each sequence of the dataset a percentage of the amount of frames (from $0.1\%$ to $10\%$) in the other sequences is randomly sampled, \ie this amount of images is used in the semi-supervised learning algorithm with its corresponding nodes in an unseen scheme. The second experiment constructs a graph with the whole UCSD and the challenges baseline, dynamic background, PTZ, and shadow from the CDNet2014 database. In this case, a percentage of the frames is extracted from each sequence (from $0.1\%$ to $10\%$), excluding all the sequences of UCSD, resulting in a graph with $N=104261$. Finally, the reconstruction algorithm classifies the non-labeled nodes in both experiments. These procedures are repeated $5$ times in both experiments for each sequence and each sampling density (this procedure is called Monte Carlo cross-validation in the literature of machine learning). Notice for example that CDNet2014 has 53 sequences, \ie for one sampling density the reconstruction is performed $53\times 5$ times. Also, the set of sampling densities in the first experiment is $\mathcal{M}=\{ 0.1\%, 0.5\%, 1\%, 2\%, 4\%, 6\%, 8\%, 10\% \}$, \ie the total amount of reconstructions for the CDNet2014 is $53 \times 5 \times \vert \mathcal{M} \vert = 2120$. For the sake of clarification, the metrics in both experiments are computed on unseen videos.

\subsection{Implementation Details}

The instance segmentation algorithm was implemented using Pytorch and Detectron2 \cite{wu2019detectron2}. The algorithm for reconstruction of graph signals was implemented using the graph signal processing toolbox \cite{perraudin2014gspbox}. For the comparison, most of the background subtraction methods were implemented with the BGSLibrary \cite{sobral2014bgs} and the LRSLibrary \cite{lrslibrary2015} with default parameters by reference.

\section{Results and Discussion}
\label{sec:results_and_discussion}

GraphBGS is compared, either in the first or in the second experiment, with the following methods: FTSG \cite{wang2014static}, SuBSENSE \cite{st2014subsense}, PAWCS \cite{st2015self}, WisenetMD \cite{lee2019wisenetmd}, IUTIS-5 \cite{bianco2017combination}, SemanticBGS \cite{braham2017semantic}, BSUV-Net \cite{tezcan2020bsuv}, DECOLOR \cite{zhou2012moving}, ViBe \cite{barnich2010vibe}, 3WD \cite{oreifej2012simultaneous}, GRASTA \cite{he2012incremental}, ROSL \cite{shu2014robust}, ADMM \cite{parikh2014proximal}, and OR1MP \cite{wang2015orthogonal}. The results of FgSegNet v2 \cite{lim2019learning} with unseen videos evaluation are also reported for reference in CDNet2014. The results of GraphBGS are computed as the best results of F-measure for each sequence in the experiments.

Table \ref{tbl:visual_results} shows some visual results of GraphBGS and other methods. Furthermore, Tables \ref{tbl:CDNet2014_results} and \ref{tbl:UCSD_results} show the comparison of GraphBGS with the state-of-the-art methods in CDNet2014 and UCSD, respectively. The results of challenges turbulence and night videos are not displayed since Cascade Mask R-CNN fails to segment some of the videos. As a consequence, it is not possible to detect the moving objects in some sequences for the semi-supervised learning algorithm. The results of BSUV-Net and FgSegNet v2 in Table \ref{tbl:CDNet2014_results} come from \cite{tezcan2020bsuv}, and the metrics of the other methods come from the change detection website \cite{wang2014cdnet}.

\begin{table*}
\centering
\caption{Comparisons of average F-measure over nine challenges of CDNet2014. GraphBGS is compared with unsupervised and supervised algorithms in background subtraction. The best and second best performing method for each challenge are shown in {\color{red}\textbf{red}} and {\color{blue}\textbf{blue}}, respectively.}
\label{tbl:CDNet2014_results}
\begin{tabular}{l|ccccccccc}
\hline 
Challenge & {FTSG} & {SuBSENSE} & {PAWCS} & {WisenetMD} & {IUTIS-5} & {SemanticBGS} & {FgSegNet v2} & {BSUV-Net} & {GraphBGS} \\
\hline
Bad Weather & 0.8228 & 0.8619 & 0.8152 & 0.8616 & 0.8248 & 0.8260 & 0.7952 & \color{blue}\textbf{0.8713} & \color{red}\textbf{0.9085} \\
Baseline & 0.9330 & 0.9503 & 0.9397 & 0.9487 & 0.9567 & \color{blue}\textbf{0.9604} & 0.6926 & \color{red}\textbf{0.9693} & 0.9535 \\
Camera Jitter & 0.7513 & 0.8152 & 0.8137 & 0.8228 & 0.8332 & \color{blue}\textbf{0.8388} & 0.4266 & 0.7743 & \color{red}\textbf{0.8826} \\
Dynamic-B & 0.8792 & 0.8177 & \color{blue}\textbf{0.8938} & 0.8376 & 0.8902 & \color{red}\textbf{0.9489} & 0.3634 & 0.7967 & 0.8353 \\
I-O Motion & \color{red}\textbf{0.7891} & 0.6569 & 0.7764 & 0.7264 & 0.7296 & \color{blue}\textbf{0.7878} & 0.2002 & 0.7499 & 0.5036 \\
Low-F rate & 0.6259 & 0.6445 & 0.6588 & 0.6404 & \color{blue}\textbf{0.7743} & \color{red}\textbf{0.7888} & 0.2482 & 0.6797 & 0.6022 \\
PTZ & 0.3241 & 0.3476 & 0.4615 & 0.3367 & 0.4282 & 0.5673 & 0.3503 & \color{blue}\textbf{0.6282} & \color{red}\textbf{0.7993} \\
Shadow & 0.8832 & 0.8986 & 0.8913 & 0.8984 & 0.9084 & \color{blue}\textbf{0.9478} & 0.5295 & 0.9233 & \color{red}\textbf{0.9712} \\
Thermal & 0.7768 & 0.8171 & 0.8324 & 0.8152 & 0.8303 & 0.8219 & 0.6038 & \color{blue}\textbf{0.8581} & \color{red}\textbf{0.8594} \\
\hline
Overall & 0.7539 & 0.7566 & 0.7870 & 0.7653 & 0.7923 & \color{red}\textbf{0.8320} & 0.4158 & 0.8056 & \color{blue}\textbf{0.8128} \\
\hline
\end{tabular}
\end{table*}

\begin{table*}[h]
\centering
\caption{Comparison of F-measure results over the videos of UCSD background subtraction dataset. The best and second best performing methods for each video are shown in {\color{red}\textbf{red}} and {\color{blue}\textbf{blue}}, respectively.}
\label{tbl:UCSD_results}
\begin{tabular}{l|ccccccccccc}
\hline
Sequence & {DECOLOR} & {ViBe} & {3WD} & {GRASTA} & {SuBSENSE} & {ROSL} & {ADMM} & {OR1MP} & {BSUV-Net} & {GraphBGS} \\
\hline
Birds & 0.1457 & 0.3354 & 0.1308 & 0.1320 & \color{blue}\textbf{0.4832} & 0.1478 & 0.0227 & 0.1394 & 0.2625 & \color{red}\textbf{0.7143} \\
Boats & 0.2179 & 0.1854 & 0.1576 & 0.0678 & 0.4550 & 0.1637 & 0.1212 & 0.1100 & \color{blue}\textbf{0.6621} & \color{red}\textbf{0.7594} \\
Bottle & 0.4765 & 0.4512 & 0.1364 & 0.1159 & 0.6570 & 0.2069 & {\color{blue}\textbf{0.6589}} & 0.1795 & 0.5039 & \color{red}\textbf{0.8741} \\
Chopper & 0.6214 & 0.4930 & 0.3171 & 0.0842 & \color{blue}\textbf{0.6723} & 0.2920 & 0.1250 & 0.2653 & 0.3020 & \color{red}\textbf{0.6956} \\
Cyclists & 0.2224 & 0.1211 & 0.1003 & 0.1243 & 0.1445 & 0.1366 & 0.1093 & 0.1242 & \color{blue}\textbf{0.4138} & \color{red}\textbf{0.7330} \\
Flock & 0.2943 & 0.2306 & 0.2007 & 0.1612 & 0.2492 & \color{blue}\textbf{0.3409} & 0.1088 & 0.2605 & 0.0025 & \color{red}\textbf{0.5872} \\
Freeway & \color{blue}\textbf{0.5229} & 0.4002 & 0.5028 & 0.0814 & \color{red}\textbf{0.5518} & 0.3875 & 0.0816 & 0.1549 & 0.1185 & 0.3491 \\
Hockey & 0.3449 & 0.4195 & 0.2789 & 0.3149 & 0.3611 & 0.4106 & 0.2981 & 0.4296 & \color{blue}\textbf{0.6908} & \color{red}\textbf{0.7608} \\
Jump & 0.3135 & 0.2636 & 0.2481 & 0.4175 & 0.2295 & 0.4198 & 0.0609 & 0.3073 & \color{red}\textbf{0.8697} & \color{blue}\textbf{0.7727} \\
Landing & 0.0640 & 0.0433 & 0.0457 & 0.0414 & 0.0026 & 0.0506 & \color{blue}\textbf{0.0826} & 0.0442 & 0.0012 & \color{red}\textbf{0.1822} \\
Ocean & 0.1315 & 0.1648 & 0.2055 & 0.1144 & 0.2533 & 0.1422 & 0.1809 & 0.1252 & \color{blue}\textbf{0.5335} & \color{red}\textbf{0.8593} \\
Peds & \color{blue}\textbf{0.7942} & 0.5257 & 0.7536 & 0.4653 & 0.5154 & 0.7418 & 0.6667 & 0.4297 & 0.6738 & \color{red}\textbf{0.8465} \\
Skiing & \color{blue}\textbf{0.3473} & 0.1441 & 0.1981 & 0.0927 & 0.2482 & 0.1942 & 0.0519 & 0.1791 & 0.0602 & \color{red}\textbf{0.5333}\\
Surf & {\color{blue}\textbf{0.0647}} & 0.0462 & 0.0579 & 0.0523 & 0.0467 & 0.0453 & 0.0162 & 0.0317 & 0 & \color{red}\textbf{0.5302} \\
Surfers & 0.1959 & 0.1189 & 0.0962 & 0.0742 & 0.1393 & 0.1184 & 0.1950 & 0.1044 & \color{blue}\textbf{0.4776} & \color{red}\textbf{0.6400} \\
Trafic & \color{blue}\textbf{0.2732} & 0.1445 & 0.2032 & 0.0368 & 0.1165 & 0.1042 & 0.1044 & 0.0882 & 0 & \color{red}\textbf{0.5722} \\
\hline
Overall & 0.3144 & 0.2555 & 0.2271 & 0.1485 & 0.3203 & 0.2439 & 0.1803 & 0.1858 & \color{blue}\textbf{0.3483} & \color{red}\textbf{0.6506} \\
\hline
\end{tabular}
\end{table*}

In the first experiment, GraphBGS outperforms all the state-of-the-art methods in the challenges: bad weather, camera jitter, PTZ, shadow, and thermal, while having the second-best performance in overall. 
On the other hand, GraphBGS does not have good performance in the challenges of intermittent object motion and low frame rate. 
Due to the nature of these two challenges, the objects in these sequences particularly show non-consistent values for the optical flow estimation between two consecutive frames. 
This problem can be alleviated by the introduction of more powerful feature representation, which we leave for future work.

GraphBGS also shows the best results in almost all the videos of UCSD as shown in Table \ref{tbl:UCSD_results}. 
The results in Tables \ref{tbl:CDNet2014_results} and \ref{tbl:UCSD_results} suggest that GraphBGS works well for both static and camera moving sequences. 
Even though GraphBGS is a simple method (temporal median filter, optical flow estimation, local binary patterns) compared to CNNs where millions of parameters are learned, the results of GraphBGS are promising. 
This method could be the first of several semi-supervised learning algorithms that exploit concepts of GSP in the field of background subtraction/moving object detection, and perhaps other domains such as video object segmentation \cite{tokmakov2017learning}, tracking \cite{xu2020train}, among others.


\section{Conclusions}
\label{sec:conclusions}

This paper introduces an algorithm for background subtraction based on the theory of reconstruction of graph signals, and it is dubbed as Graph BackGround Subtraction (GraphBGS). The pipeline of our algorithm involves a Cascade Mask R-CNN as instance segmentation; median filter as background initialization; optical flow, intensity, and texture features for the representation of the nodes in a graph; a k-nearest neighbors for the construction of the graph; and finally a semi-supervised learning algorithm based on reconstruction of graph signals. Unlike most methods of the state-of-the-art, GraphBGS works well for both static and moving camera sequences. Our algorithm outperforms state-of-the-art methods in several challenging conditions of CDNet2014 and UCSD databases.

The present work opens up some new questions for future research. What is the role of graphs and the structure of a given dataset to improve the generalization of current deep learning methods? What is the relationship between the sampling of graph signals and the problem of background subtraction? Perhaps, concepts of GSP such as sampling \cite{anis2018sampling}, learning graphs from data \cite{kalofolias2016learn}, and graph convolutional neural networks \cite{kipf2017semi} could improve the performance of state-of-the-art methods in the field of background subtraction.

\appendices

\section{Graph Signal Processing}
\label{app:GSP}

Graphs provide the ability to model interactions of data residing on irregular and complex structures. Social, financial, and sensor networks are examples of data that can be modeled on graphs. Graph Signal Processing (GSP) is thus an emerging field that extends the concepts of classical digital signal processing to signals supported on graphs \cite{sandryhaila2013discrete,shuman2013emerging,sandryhaila2014big,ortega2018graph}. A graph signal is defined as a function over the nodes of a graph. The sampling and recovery of graph signals are, as expected, fundamental tasks in GSP that have received considerable attention recently. Naturally, the mathematics of sampling theory and spectral graph theory have been combined leading to generalized Nyquist sampling principles for graphs \cite{anis2014towards,shomorony2014sampling,chen2015discrete,anis2016efficient}. In general, these sampling methods are based on the assumption that the underlying graph spectral decompositions are known and available in advance \cite{anis2014towards,chen2015discrete}. This appendix is intended to explain some basic concepts of GSP for novices.

\subsection{Basic Notions}

Let $G=(\mathcal{V},\mathcal{E})$ be an undirected, weighted, and connected graph. $\mathbf{W} \in \mathbb{R}^{N\times N}$ is the adjacency matrix of the graph. 
A graph signal is a function $y:\mathcal{V}\to \mathbb{R}$ defined on the nodes of $G$, and can be represented as $\mathbf{y} \in \mathbb{R}^N$. 
Furthermore, $\mathbf{D} \in \mathbb{R}^{N\times N}$ is the degree matrix of the graph.
Likewise, the combinatorial Laplacian operator $\mathbf{L}$ is a positive semi-definite matrix.
Since $\mathbf{L}$ is positive semi-definite, all the eigenvalues are real non-negative such that: $0=\lambda_1 \leq \lambda_2 \leq \dots \leq \lambda_N$, while there exists a full set of orthogonal eigenvectors $\{ \mathbf{u}_1,\mathbf{u}_2,\dots,\mathbf{u}_N\}$. The graph Fourier basis of $G$ is defined by the spectral decomposition of the Laplacian matrix such that:
\begin{equation}
    \mathbf{L} = \mathbf{U}\mathbf{\Lambda}\mathbf{U}^{\mathsf{T}},
    \label{eqn:eigenbasis}
\end{equation}
where $\mathbf{U}=[\mathbf{u}_1,\mathbf{u}_2,\dots,\mathbf{u}_N]$ and $\mathbf{\Lambda}=\diag(\lambda_1,\lambda_2,\dots,\lambda_N)$. 
The eigenvalues of $\mathbf{L}$ provide a natural way to order the graph Fourier basis in terms of frequency. Therefore, the Graph Fourier Transform (GFT) of a graph signal is defined as $\mathbf{\hat{y}}=\mathbf{U}^{{\mathsf{T}}}\mathbf{y}$, and the inverse GFT is given by $\mathbf{y} = \mathbf{U}\mathbf{\hat{y}}$. More generally, the relationships between the Laplacian eigenvalues and eigenvectors, as well as the structure of a graph are part of the domain of mathematics known as spectral graph theory \cite{chung1997spectral}.

Using the notions of frequency in graphs, Pesenson \cite{pesenson2008sampling} defined the space of all $\omega$\textit{-bandlimited} signals as $PW_{\omega}(G)=\spanmath(\mathbf{U}_{\rho}:\lambda_{\rho} \leq \omega)$, where $\mathbf{U}_{\rho}$ represents the first $\rho$ eigenvectors of $\mathbf{L}$, and $PW_{\omega}(G)$ is known as the Paley-Wiener space.
\begin{definition}
    \label{dfn:bandlimited_signal}
    A graph signal $\mathbf{y}$ is called bandlimited if $\exists~\rho \in \{ 1,2,\dots,N-1 \}$ such that its GFT satisfies $\mathbf{\hat{y}}(i)=0~\forall~i > \rho$.
\end{definition}
Using Definition \ref{dfn:bandlimited_signal} and the notions of Paley-Wiener spaces, a graph signal $\mathbf{y}$ has cutoff frequency $\omega$, and bandwidth $\rho$ if $\mathbf{y} \in PW_{\omega}(G)$.

\subsection{Sampling of Graph Signals}

Given the notion of bandlimitedness in terms of $PW_{\omega}(G)$, the next step is to find a bound for the minimum sampling rate allowing perfect recovery of graph signals $\mathbf{y} \in PW_{\omega}(G)$. 
The sampling rate is defined in terms of a subset of nodes $\mathcal{S} \subset \mathcal{V}$ in $G$ with $\mathcal{S}=\{s_1,s_2,\dots,s_m\}$, where $m = \vert \mathcal{S}\vert \leq N$ is the number of sampled nodes. The sampled graph signal is defined as $\mathbf{y}(\mathcal{S})=\mathbf{My}$, where $\mathbf{M}$ is a binary decimation matrix whose entries are given by $\mathbf{M} = [\boldsymbol{\delta}_{s_1},\dots,\boldsymbol{\delta}_{s_m}]^{\mathsf{T}}$. Let $\mathbf{\Phi} \in \mathbb{R}^{N \times m}$ be an interpolation operator and let $\tilde{\mathbf{y}} = \mathbf{\Phi My}$ be the reconstructed graph signal, perfect recovery happens when $\mathbf{\Phi M}$ is the identity matrix. This is not possible in general since $\rankmath(\mathbf{\Phi M}) \leq m \leq N$. However, perfect recovery from $\mathbf{y}(\mathcal{S})$ is possible if the sampling size $\vert \mathcal{S}\vert$ is greater than or equal to $\rho$, \ie $|\mathcal{S}| \geq \rho$ \cite{chen2015discrete}.

\begin{theorem}[Chen's theorem \cite{chen2015discrete}]
    \label{trm:perfect_reconstruction}
    Let $\mathbf{M}$ satisfy $\rankmath(\mathbf{MU}_{\rho})=\rho$. For all $\mathbf{y} \in PW_{\omega}(G)$, perfect recovery, \ie $\mathbf{y}=\mathbf{\mathbf{\Phi My}}$, is achieved by choosing:
    \begin{equation}
        \mathbf{\Phi} = \mathbf{U}_{\rho}\mathbf{V},
        \label{eqn:interpo_perfect_rec}
    \end{equation}
    with $\mathbf{V\Phi U}_{\rho}$ a $\rho\times \rho$ identity matrix.
    \\
    Proof: see \cite{chen2015discrete}.
\end{theorem}
Theorem \ref{trm:perfect_reconstruction} states that perfect reconstruction of graph signal from its samples is possible when $\mathbf{y}$ lies in $PW_{\omega}(G)$, and the number of samples is at least $\rho$. Then, perfect reconstruction is achieved by choosing the interpolation operator as in Eqn. (\ref{eqn:interpo_perfect_rec}).

\subsection{Reconstruction of Graph Signals}

Reconstruction of graph signals has been studied assuming a bandlimited constraint, \ie $\mathbf{y} \in PW_{\omega}(G)$. Given the sampled graph signal $\mathbf{y}(\mathcal{S})=\mathbf{My}$, a common approach to obtain a reconstructed version of $\mathbf{y}$ is given by:
\begin{equation}
    \mathbf{\tilde{y}} = \argmin_{\mathbf{z}\in \spanmath(\mathbf{U}_{\rho})} || \mathbf{Mz} - \mathbf{y}(\mathcal{S})||_2^2 = \mathbf{U}_{\rho}(\mathbf{MU}_{\rho})^{\dagger}\mathbf{y}(\mathcal{S}),
    \label{eqn:classic_recovery_method}
\end{equation}
where $\mathbf{U}_{\rho}=[\mathbf{u}_1,\mathbf{u}_2,\dots,\mathbf{u}_{\rho}]$ is the matrix formed of the first ${\rho}$ graph's eigenvectors, and $(\mathbf{MU}_{\rho})^{\dagger}$ is the Moore-Penrose pseudo inverse of $\mathbf{MU}_{\rho}$. Notice that $(\mathbf{MU}_{\rho})^{\dagger}$ is precisely $\mathbf{V}$ in Eqn. (\ref{eqn:interpo_perfect_rec}) from Theorem \ref{trm:perfect_reconstruction}. This approach has the significant drawback that if the underlying signal is not strictly bandlimited, the reconstruction losses accuracy rapidly. In practice, graph signals are only approximately bandlimited \cite{anis2016efficient,chen2016signal}, and the calculation of eigendecompositions may be computationally expensive.

Puy \etal \cite{puy2016random} proposed to estimate $\mathbf{\tilde{y}}$ by solving the following problem:
\begin{equation}
    \mathbf{\tilde{y}} = \argmin_{\mathbf{z}} \Vert \mathbf{P}^{-\frac{1}{2}}(\mathbf{Mz-y}(\mathcal{S})) \Vert_2 + \eta \mathbf{z}^{\mathsf{T}}g(\mathbf{L})\mathbf{z},
    \label{eqn:recovery_puy}
\end{equation}
where $\eta \in \mathbb{R}^+$, $g:\mathbb{R} \to \mathbb{R}$ is a nonnegative and nondecreasing polynomial function, and $\mathbf{P} \in \mathbb{R}^{N\times N}$ is a random matrix that is designed jointly with $\mathbf{M}$. The computation of $\mathbf{P}$ requires the spectral decomposition of $\mathbf{L}$; however, Puy \etal proposed an estimation of $\mathbf{P}$ using fast filtering techniques on graphs. The solution of Eqn. (\ref{eqn:recovery_puy}) is given by:
\begin{equation}
    \mathbf{\tilde{y}} = (\mathbf{M}^{\mathsf{T}}\mathbf{P}^{-1}\mathbf{M}+\eta g(\mathbf{L}))^{-1}\mathbf{M}^{\mathsf{T}}\mathbf{P}^{-1}\mathbf{y}({\mathcal{S}}).
    \label{eqn:puy_solution}
\end{equation}
The parameter $\eta$ and the function $g(\cdot)$ in Eqn. (\ref{eqn:puy_solution}) should be selected empirically.

Finally, the readers who are interested in a more detailed and in deep explanations of GSP are referred to the review papers \cite{shuman2013emerging,ortega2018graph}, as well as the papers \cite{chen2015discrete,marques2015sampling,tsitsvero2016signals,anis2016efficient,romero2016kernel,parada2019blue}.

\section{Proof Theorem Perturbation Matrix}
\label{app:perturbation_matrix}

\begin{proof}

For simplicity, and without loss of generality consider the combinatorial Laplacian matrix $\mathbf{L} \in \mathbb{R}^{N \times N}$ of $G$ such that $\text{rank}(\mathbf{L}) = N-1$, \ie $\mathbf{L}$ is ill-conditioned. Since $\mathbf{L}$ does not have full rank $\exists \text{ } \mathbf{x} \neq \mathbf{0} \text{ } \vert \text{ } \mathbf{Lx} = \mathbf{0}$. Then:
\begin{gather}
    \nonumber
    \mathbf{(L+\Psi)x} = \mathbf{\Psi x}\\
    \nonumber
    \Vert \mathbf{(L+\Psi)x} \Vert_2 = \Vert \mathbf{\Psi x} \Vert_2\\
    \label{eqn:relation_L_psi}
    \frac{\Vert \mathbf{(L+\Psi)x} \Vert_2}{\Vert \mathbf{x} \Vert_2} = \frac{\Vert \mathbf{\Psi x} \Vert_2}{\Vert \mathbf{x} \Vert_2},
\end{gather}
where $\mathbf{\Psi} \in \mathbb{R}^{N \times N}$ is a perturbation matrix.

\begin{lemma}
If $\Vert \cdot \Vert$ is a matrix norm induced by a vector norm $\Vert \cdot \Vert$, then:
\begin{equation}
    \Vert \mathbf{A x} \Vert \leq \Vert \mathbf{A} \Vert \cdot \Vert \mathbf{x} \Vert.
\end{equation}
\label{lma:matrix_norm}
\end{lemma}

\begin{proof}
Since $\Vert \mathbf{A} \Vert = \max_{\mathbf{x} \neq \boldsymbol{0}} \frac{\Vert \mathbf{A x} \Vert}{\Vert \mathbf{x} \Vert}$, for an arbitrary $\mathbf{y} \in \mathbb{R}^N$, with $\mathbf{y} \neq \boldsymbol{0}$:
\begin{gather}
    \nonumber
    \Vert \mathbf{A} \Vert = \max_{\mathbf{x} \neq \boldsymbol{0}} \frac{\Vert \mathbf{A x} \Vert}{\Vert \mathbf{x} \Vert} \geq \frac{\Vert \mathbf{A y} \Vert}{\Vert \mathbf{y} \Vert}\\
    \label{eqn:matrix_norm_inequality}
    \rightarrow \Vert \mathbf{A y} \Vert \leq \Vert \mathbf{A} \Vert \cdot \Vert \mathbf{y} \Vert. 
\end{gather}
This also holds true for $\mathbf{y} = \boldsymbol{0}$.
\end{proof}

Using Eqn. \ref{eqn:relation_L_psi} and Lemma \ref{lma:matrix_norm} we have:
\begin{equation}
    \frac{\Vert \mathbf{(L+\Psi)x} \Vert_2}{\Vert \mathbf{x} \Vert_2} = \frac{\Vert \mathbf{\Psi x} \Vert_2}{\Vert \mathbf{x} \Vert_2} \leq \Vert \mathbf{\Psi} \Vert_2 = \sigma_{\text{max}} (\mathbf{\Psi}),
    \label{eqn:inequality_max_singular_value}
\end{equation}
where $\sigma_{\text{max}} (\mathbf{\Psi})$ is the maximum singular value of $\mathbf{\Psi}$.

\begin{lemma}
    Let $\mathbf{L+\Psi} \in \mathbb{R}^{N\times N}$ be full rank, then:
    \begin{equation}
        \inf_{\mathbf{x} \neq \boldsymbol{0}} \frac{\Vert \mathbf{(L+\Psi)x} \Vert_2}{\Vert \mathbf{x} \Vert_2} = \sigma_{\text{min}}(\mathbf{L+\Psi}),
    \end{equation}
    where $\sigma_{\text{min}}(\mathbf{L+\Psi})$ is the minimum singular value of $\mathbf{L+\Psi}$
    \label{lma:inequality_min_singular_value}
\end{lemma}

\begin{proof}
    Using the singular value decomposition $\mathbf{L+\Psi} = \mathbf{\tilde{U}\Sigma V}^{\mathsf{T}}$, and the fact that $\mathbf{\tilde{U}}$ and $\mathbf{V}$ are unitary matrices we have:
    \begin{gather}
        \nonumber
        \inf_{\mathbf{x} \neq \boldsymbol{0}} \frac{\Vert \mathbf{(L+\Psi)x} \Vert_2}{\Vert \mathbf{x} \Vert_2} = \inf_{\mathbf{x} \neq \boldsymbol{0}} \frac{\Vert \mathbf{\tilde{U}\Sigma V}^{\mathsf{T}}\mathbf{x} \Vert_2}{\Vert \mathbf{x} \Vert_2} = \inf_{\mathbf{x} \neq \boldsymbol{0}} \frac{\Vert \mathbf{\Sigma V}^{\mathsf{T}}\mathbf{x} \Vert_2}{\Vert \mathbf{x} \Vert_2}\\
        \nonumber
        \rightarrow \inf_{\mathbf{x} \neq \boldsymbol{0}} \frac{\Vert \mathbf{\Sigma V}^{\mathsf{T}}\mathbf{x} \Vert_2}{\Vert \mathbf{x} \Vert_2} = \inf_{\mathbf{y} \neq \boldsymbol{0}} \frac{\Vert \mathbf{\Sigma}\mathbf{y} \Vert_2}{\Vert \mathbf{Vy} \Vert_2} = \inf_{\mathbf{y} \neq \boldsymbol{0}} \frac{\Vert \mathbf{\Sigma}\mathbf{y} \Vert_2}{\Vert \mathbf{y} \Vert_2}\\
        \nonumber
        \rightarrow \inf_{\mathbf{y} \neq \boldsymbol{0}} \frac{\Vert \mathbf{\Sigma}\mathbf{y} \Vert_2}{\Vert \mathbf{y} \Vert_2}=\inf_{\mathbf{y} \neq \boldsymbol{0}} \frac{ (\sum_i \vert \sigma_i y_i \vert^2)^{\frac{1}{2}}}{(\sum_i \vert y_i \vert^2)^{\frac{1}{2}}},
    \end{gather}
    where we used the change of variable $\mathbf{y} = \mathbf{V}^{\mathsf{T}}\mathbf{x}$, and $\sigma_i$ is the $i$-th singular value of $\mathbf{L+\Psi}$. Finally,
    \begin{gather}
        \nonumber
        \frac{(\sum_i \vert \sigma_i y_i \vert^2)^{\frac{1}{2}}}{(\sum_i \vert y_i \vert^2)^{\frac{1}{2}}} \geq \sigma_{\text{min}}(\mathbf{L+\Psi})\\
        \nonumber
        \rightarrow \inf_{\mathbf{x} \neq \boldsymbol{0}} \frac{\Vert \mathbf{(L+\Psi)x} \Vert_2}{\Vert \mathbf{x} \Vert_2} = \sigma_{\text{min}}(\mathbf{L+\Psi}),
    \end{gather}
    were $\sigma_{\text{min}}(\mathbf{L+\Psi}) = \sigma_N(\mathbf{L+\Psi})$.
\end{proof}

Using Eqn. \ref{eqn:inequality_max_singular_value}, and lemma \ref{lma:inequality_min_singular_value} we have:
\begin{gather}
    \nonumber
    \sigma_{\text{min}}(\mathbf{L+\Psi}) \leq \frac{\Vert \mathbf{(L+\Psi)x} \Vert_2}{\Vert \mathbf{x} \Vert_2} \leq \sigma_{\text{max}}(\mathbf{\Psi})\\
    \nonumber
    \rightarrow \sigma_{\text{min}}(\mathbf{L+\Psi}) \leq \sigma_{\text{max}}(\mathbf{\Psi})\\
    \nonumber
    1 \leq \frac{\sigma_{\text{max}}(\mathbf{\Psi})}{\sigma_{\text{min}}(\mathbf{L+\Psi})}\\
    \nonumber
    \sigma_{\text{max}}(\mathbf{L+\Psi}) \leq \sigma_{\text{max}}(\mathbf{\Psi})\kappa(\mathbf{L+\Psi})\\
    \label{eqn:lower_bound_condition_number_proof}
    \frac{\sigma_{\text{max}}(\mathbf{L+\Psi})}{\sigma_{\text{max}}(\mathbf{\Psi})} \leq \kappa(\mathbf{L+\Psi}),
\end{gather}
where $\kappa(\mathbf{L+\Psi})=\sigma_{\text{max}}(\mathbf{L+\Psi})/\sigma_{\text{min}}(\mathbf{L+\Psi})$ is the condition number in the $\ell_2$-norm of $\mathbf{L+\Psi}$.

Now, using the triangle inequality of matrix norms
\begin{gather}
    \nonumber
    \Vert \mathbf{L+\Psi} \Vert_2 \leq \Vert \mathbf{L} \Vert_2 + \Vert \mathbf{\Psi} \Vert_2\\
    \nonumber
    \rightarrow \sigma_{\text{max}}(\mathbf{L+\Psi}) \leq \sigma_{\text{max}}(\mathbf{L}) + \sigma_{\text{max}}(\mathbf{\Psi})\\
    \label{eqn:upper_bound_condition_number_proof}
    \kappa(\mathbf{L+\Psi}) \leq \frac{\sigma_{\text{max}}(\mathbf{L})+\sigma_{\text{max}}(\mathbf{\Psi})}{\sigma_{min}(\mathbf{L+\Psi})}.
\end{gather}

Finally, using Eqn. \ref{eqn:lower_bound_condition_number_proof} and \ref{eqn:upper_bound_condition_number_proof} we have:
\begin{equation}
    \frac{\sigma_{\text{max}}(\mathbf{L+\Psi})}{\sigma_{\text{max}}(\mathbf{\Psi})} \leq \kappa(\mathbf{L+\Psi}) \leq \frac{\sigma_{\text{max}}(\mathbf{L})+\sigma_{\text{max}}(\mathbf{\Psi})}{\sigma_{min}(\mathbf{L+\Psi})}.
    \label{eqn:lower_upper_bound_condition_number}
\end{equation}
\end{proof}

\bibliographystyle{IEEEtran}
\bibliography{bibfile}

\begin{thebibliography}{10}
\providecommand{\url}[1]{#1}
\csname url@samestyle\endcsname
\providecommand{\newblock}{\relax}
\providecommand{\bibinfo}[2]{#2}
\providecommand{\BIBentrySTDinterwordspacing}{\spaceskip=0pt\relax}
\providecommand{\BIBentryALTinterwordstretchfactor}{4}
\providecommand{\BIBentryALTinterwordspacing}{\spaceskip=\fontdimen2\font plus
\BIBentryALTinterwordstretchfactor\fontdimen3\font minus
  \fontdimen4\font\relax}
\providecommand{\BIBforeignlanguage}[2]{{%
\expandafter\ifx\csname l@#1\endcsname\relax
\typeout{** WARNING: IEEEtran.bst: No hyphenation pattern has been}%
\typeout{** loaded for the language `#1'. Using the pattern for}%
\typeout{** the default language instead.}%
\else
\language=\csname l@#1\endcsname
\fi
#2}}
\providecommand{\BIBdecl}{\relax}
\BIBdecl

\bibitem{garcia2020background}
B.~Garcia-Garcia, T.~Bouwmans, and A.~J.~R. Silva, ``Background subtraction in
  real applications: Challenges, current models and future directions,''
  \emph{Computer Science Review}, vol.~35, 2020.

\bibitem{bouwmans2014traditional}
T.~Bouwmans, ``Traditional and recent approaches in background modeling for
  foreground detection: An overview,'' \emph{Computer science review}, vol.~11,
  pp. 31--66, 2014.

\bibitem{bouwmans2017decomposition}
T.~Bouwmans, A.~Sobral, S.~Javed, S.~K. Jung, and E.~H. Zahzah, ``Decomposition
  into low-rank plus additive matrices for background/foreground separation: A
  review for a comparative evaluation with a large-scale dataset,''
  \emph{Computer Science Review}, vol.~23, pp. 1--71, 2017.

\bibitem{tezcan2020bsuv}
O.~Tezcan, P.~Ishwar, and J.~Konrad, ``{BSUV-Net}: A fully-convolutional neural
  network for background subtraction of unseen videos,'' in \emph{WACV}.\hskip
  1em plus 0.5em minus 0.4em\relax IEEE, 2020, pp. 2774--2783.

\bibitem{wang2014cdnet}
Y.~Wang, P.~M. Jodoin, F.~Porikli, J.~Konrad, Y.~Benezeth, and P.~Ishwar,
  ``{CDnet} 2014: An expanded change detection benchmark dataset,'' in
  \emph{CVPRW}.\hskip 1em plus 0.5em minus 0.4em\relax IEEE, 2014, pp.
  387--394.

\bibitem{oreifej2012simultaneous}
O.~Oreifej, X.~Li, and M.~Shah, ``Simultaneous video stabilization and moving
  object detection in turbulence,'' \emph{IEEE Transactions on Pattern Analysis
  and Machine Intelligence}, vol.~35, no.~2, pp. 450--462, 2012.

\bibitem{pham2014detection}
D.~S. Pham, O.~Arandjelovi{\'c}, and S.~Venkatesh, ``Detection of dynamic
  background due to swaying movements from motion features,'' \emph{IEEE
  Transactions on Image Processing}, vol.~24, no.~1, pp. 332--344, 2014.

\bibitem{li2018fusion}
S.~Li, D.~Florencio, W.~Li, Y.~Zhao, and C.~Cook, ``A fusion framework for
  camouflaged moving foreground detection in the wavelet domain,'' \emph{IEEE
  Transactions on Image Processing}, vol.~27, no.~8, pp. 3918--3930, 2018.

\bibitem{bouwmans2014robust}
T.~Bouwmans and E.~H. Zahzah, ``Robust {PCA} via principal component pursuit: A
  review for a comparative evaluation in video surveillance,'' \emph{Computer
  Vision and Image Understanding}, vol. 122, pp. 22--34, 2014.

\bibitem{bouwmans2019deep}
T.~Bouwmans, S.~Javed, M.~Sultana, and S.~K. Jung, ``Deep neural network
  concepts for background subtraction: A systematic review and comparative
  evaluation,'' \emph{Neural Networks}, 2019.

\bibitem{lim2019learning}
L.~A. Lim and H.~Y. Keles, ``Learning multi-scale features for foreground
  segmentation,'' \emph{Pattern Analysis and Applications}, Aug 2019.

\bibitem{choo2018multi}
S.~Choo, W.~Seo, D.~Jeong, and N.~I. Cho, ``Multi-scale recurrent
  encoder-decoder network for dense temporal classification,'' in
  \emph{ICPR}.\hskip 1em plus 0.5em minus 0.4em\relax IEEE, 2018, pp. 103--108.

\bibitem{vapnik2013nature}
V.~Vapnik, \emph{The nature of statistical learning theory}.\hskip 1em plus
  0.5em minus 0.4em\relax Springer science \& business media, 2013.

\bibitem{shalev2014understanding}
S.~Shalev-Shwartz and S.~Ben-David, \emph{Understanding machine learning: From
  theory to algorithms}.\hskip 1em plus 0.5em minus 0.4em\relax Cambridge
  university press, 2014.

\bibitem{cai2019cascade}
Z.~Cai and N.~Vasconcelos, ``Cascade {R-CNN}: High quality object detection and
  instance segmentation,'' \emph{IEEE Transactions on Pattern Analysis and
  Machine Intelligence}, 2019.

\bibitem{shuman2013emerging}
D.~I. Shuman, S.~K. Narang, P.~Frossard, A.~Ortega, and P.~Vandergheynst, ``The
  emerging field of signal processing on graphs: Extending high-dimensional
  data analysis to networks and other irregular domains,'' \emph{IEEE Signal
  Processing Magazine}, vol.~30, no.~3, pp. 83--98, 2013.

\bibitem{ortega2018graph}
A.~Ortega, P.~Frossard, J.~Kova{\v{c}}evi{\'c}, J.~M.~F. Moura, and
  P.~Vandergheynst, ``Graph signal processing: Overview, challenges, and
  applications,'' \emph{Proceedings of the IEEE}, vol. 106, no.~5, pp.
  808--828, 2018.

\bibitem{chen2015discrete}
S.~Chen, R.~Varma, A.~Sandryhaila, and J.~Kova{\v{c}}evi{\'c}, ``Discrete
  signal processing on graphs: Sampling theory,'' \emph{IEEE Transactions on
  Signal Processing}, vol.~63, no.~24, pp. 6510--6523, 2015.

\bibitem{anis2016efficient}
A.~Anis, A.~Gadde, and A.~Ortega, ``Efficient sampling set selection for
  bandlimited graph signals using graph spectral proxies,'' \emph{IEEE
  Transactions on Signal Processing}, vol.~64, no.~14, pp. 3775--3789, 2016.

\bibitem{parada2019blue}
A.~Parada-Mayorga, D.~L. Lau, J.~H. Giraldo, and G.~R. Arce, ``Blue-noise
  sampling on graphs,'' \emph{IEEE Transactions on Signal and Information
  Processing over Networks}, vol.~5, no.~3, pp. 554--569, 2019.

\bibitem{mahadevan2009spatiotemporal}
V.~Mahadevan and N.~Vasconcelos, ``Spatiotemporal saliency in dynamic scenes,''
  \emph{IEEE Transactions on Pattern Analysis and Machine Intelligence},
  vol.~32, no.~1, pp. 171--177, 2009.

\bibitem{lim2018foreground}
L.~A. Lim and H.~Y. Keles, ``Foreground segmentation using convolutional neural
  networks for multiscale feature encoding,'' \emph{Pattern Recognition
  Letters}, vol. 112, pp. 256--262, 2018.

\bibitem{zeng2018multiscale}
D.~Zeng and M.~Zhu, ``Multiscale fully convolutional network for foreground
  object detection in infrared videos,'' \emph{IEEE Geoscience and Remote
  Sensing Letters}, vol.~15, no.~4, pp. 617--621, 2018.

\bibitem{hu20183d}
Z.~Hu, T.~Turki, N.~Phan, and J.~T.~L. Wang, ``A {3D} atrous convolutional long
  short-term memory network for background subtraction,'' \emph{IEEE Access},
  vol.~6, pp. 43\,450--43\,459, 2018.

\bibitem{unzueta2011adaptive}
L.~Unzueta, M.~Nieto, A.~Cort{\'e}s, J.~Barandiaran, O.~Otaegui, and
  P.~S{\'a}nchez, ``Adaptive multicue background subtraction for robust vehicle
  counting and classification,'' \emph{IEEE Transactions on Intelligent
  Transportation Systems}, vol.~13, no.~2, pp. 527--540, 2011.

\bibitem{piccardi2004background}
M.~Piccardi, ``Background subtraction techniques: A review,'' in \emph{SMC},
  vol.~4.\hskip 1em plus 0.5em minus 0.4em\relax IEEE, 2004, pp. 3099--3104.

\bibitem{lucas1981iterative}
B.~D. Lucas, T.~Kanade \emph{et~al.}, ``An iterative image registration
  technique with an application to stereo vision,'' 1981.

\bibitem{ojala2002multiresolution}
T.~Ojala, M.~Pietik{\"a}inen, and T.~M{\"a}enp{\"a}{\"a}, ``Multiresolution
  gray-scale and rotation invariant texture classification with local binary
  patterns,'' \emph{IEEE Transactions on Pattern Analysis and Machine
  Intelligence}, no.~7, pp. 971--987, 2002.

\bibitem{pesenson2008sampling}
I.~Pesenson, ``Sampling in {P}aley-{W}iener spaces on combinatorial graphs,''
  \emph{Transactions of the American Mathematical Society}, vol. 360, no.~10,
  pp. 5603--5627, 2008.

\bibitem{zhang2020resnest}
H.~Zhang, C.~Wu, Z.~Zhang, Y.~Zhu, Z.~Zhang, H.~Lin, Y.~Sun, T.~He, J.~Mueller,
  R.~Manmatha \emph{et~al.}, ``{ResNeSt}: Split-attention networks,''
  \emph{arXiv preprint arXiv:2004.08955}, 2020.

\bibitem{lin2017feature}
T.~Y. Lin, P.~Doll{\'a}r, R.~Girshick, K.~He, B.~Hariharan, and S.~Belongie,
  ``Feature pyramid networks for object detection,'' in \emph{CVPR}.\hskip 1em
  plus 0.5em minus 0.4em\relax IEEE, 2017, pp. 2117--2125.

\bibitem{lin2014microsoft}
T.~Y. Lin, M.~Maire, S.~Belongie, J.~Hays, P.~Perona, D.~Ramanan,
  P.~Doll{\'a}r, and C.~L. Zitnick, ``Microsoft {COCO}: Common objects in
  context,'' in \emph{ECCV}.\hskip 1em plus 0.5em minus 0.4em\relax Springer,
  2014, pp. 740--755.

\bibitem{achanta2012slic}
R.~Achanta, A.~Shaji, K.~Smith, A.~Lucchi, P.~Fua, and S.~S{\"u}sstrunk,
  ``{SLIC} superpixels compared to state-of-the-art superpixel methods,''
  \emph{IEEE Transactions on Pattern Analysis and Machine Intelligence},
  vol.~34, no.~11, pp. 2274--2282, 2012.

\bibitem{lu2019fast}
K.~S. Lu and A.~Ortega, ``Fast graph {Fourier} transforms based on graph
  symmetry and bipartition,'' \emph{IEEE Transactions on Signal Processing},
  vol.~67, no.~18, pp. 4855--4869, 2019.

\bibitem{javed2017background}
S.~Javed, A.~Mahmood, T.~Bouwmans, and S.~K. Jung, ``Background-foreground
  modeling based on spatiotemporal sparse subspace clustering,'' \emph{IEEE
  Transactions on Image Processing}, vol.~26, no.~12, pp. 5840--5854, 2017.

\bibitem{xu2019robust}
Z.~Xu, B.~Min, and R.~C.~C. Cheung, ``A robust background initialization
  algorithm with superpixel motion detection,'' \emph{Signal Processing: Image
  Communication}, vol.~71, pp. 1--12, 2019.

\bibitem{pesenson2009variational}
I.~Pesenson, ``Variational splines and {P}aley-{W}iener spaces on combinatorial
  graphs,'' \emph{Constructive Approximation}, vol.~29, no.~1, pp. 1--21, 2009.

\bibitem{horn2012matrix}
R.~A. Horn and C.~R. Johnson, \emph{Matrix analysis}.\hskip 1em plus 0.5em
  minus 0.4em\relax Cambridge university press, 2012.

\bibitem{perraudin2014gspbox}
N.~Perraudin, J.~Paratte, D.~Shuman, L.~Martin, V.~Kalofolias,
  P.~Vandergheynst, and D.~K. Hammond, ``{GSPBOX}: A toolbox for signal
  processing on graphs,'' \emph{arXiv preprint arXiv:1408.5781}, 2014.

\bibitem{li2016weighted}
C.~Li, X.~Wang, L.~Zhang, J.~Tang, H.~Wu, and L.~Lin, ``Weighted low-rank
  decomposition for robust grayscale-thermal foreground detection,'' \emph{IEEE
  Transactions on Circuits and Systems for Video Technology}, vol.~27, no.~4,
  pp. 725--738, 2016.

\bibitem{camplani2017rgb}
M.~Camplani, L.~Maddalena, M.~Gabriel, A.~Petrosino, and L.~Salgado, ``{RGB-D}
  dataset: Background learning for detection and tracking from {RGBD} videos,''
  in \emph{ICIAP}.\hskip 1em plus 0.5em minus 0.4em\relax Springer, 2017.

\bibitem{st2014subsense}
P.~L. St-Charles, G.~A. Bilodeau, and R.~Bergevin, ``{SuBSENSE}: A universal
  change detection method with local adaptive sensitivity,'' \emph{IEEE
  Transactions on Image Processing}, vol.~24, no.~1, pp. 359--373, 2014.

\bibitem{st2015self}
------, ``A self-adjusting approach to change detection based on background
  word consensus,'' in \emph{WACV}.\hskip 1em plus 0.5em minus 0.4em\relax
  IEEE, 2015, pp. 990--997.

\bibitem{bianco2017combination}
S.~Bianco, G.~Ciocca, and R.~Schettini, ``Combination of video change detection
  algorithms by genetic programming,'' \emph{IEEE Transactions on Evolutionary
  Computation}, vol.~21, no.~6, pp. 914--928, 2017.

\bibitem{wu2019detectron2}
Y.~Wu, A.~Kirillov, F.~Massa, W.~Y. Lo, and R.~Girshick, ``Detectron2,''
  \url{https://github.com/facebookresearch/detectron2}, 2019.

\bibitem{sobral2014bgs}
A.~Sobral and T.~Bouwmans, ``{BGS} library: A library framework for algorithm's
  evaluation in foreground/background segmentation,'' in \emph{Handbook on
  "Background Modeling and Foreground Detection for Video Surveillance'',
  Chapter 23}, 2014.

\bibitem{lrslibrary2015}
A.~Sobral, T.~Bouwmans, and E.~Zahzah, ``{LRSLibrary}: Low-rank and sparse
  tools for background modeling and subtraction in videos,'' in \emph{Robust
  Low-Rank and Sparse Matrix Decomposition: Applications in Image and Video
  Processing}.\hskip 1em plus 0.5em minus 0.4em\relax CRC Press, Taylor and
  Francis Group, 2015.

\bibitem{wang2014static}
R.~Wang, F.~Bunyak, G.~Seetharaman, and K.~Palaniappan, ``Static and moving
  object detection using flux tensor with split gaussian models,'' in
  \emph{CVPRW}.\hskip 1em plus 0.5em minus 0.4em\relax IEEE, 2014, pp.
  414--418.

\bibitem{lee2019wisenetmd}
S.~H. Lee, G.~C. Lee, J.~Yoo, and S.~Kwon, ``{WisenetMD}: motion detection
  using dynamic background region analysis,'' \emph{Symmetry}, vol.~11, no.~5,
  p. 621, 2019.

\bibitem{braham2017semantic}
M.~Braham, S.~Pi{\'e}rard, and M.~Van~Droogenbroeck, ``Semantic background
  subtraction,'' in \emph{ICIP}.\hskip 1em plus 0.5em minus 0.4em\relax IEEE,
  2017, pp. 4552--4556.

\bibitem{zhou2012moving}
X.~Zhou, C.~Yang, and W.~Yu, ``Moving object detection by detecting contiguous
  outliers in the low-rank representation,'' \emph{IEEE Transactions on Pattern
  Analysis and Machine Intelligence}, vol.~35, no.~3, pp. 597--610, 2012.

\bibitem{barnich2010vibe}
O.~Barnich and M.~Van~Droogenbroeck, ``{ViBe}: A universal background
  subtraction algorithm for video sequences,'' \emph{IEEE Transactions on Image
  processing}, vol.~20, no.~6, pp. 1709--1724, 2010.

\bibitem{he2012incremental}
J.~He, L.~Balzano, and A.~Szlam, ``Incremental gradient on the grassmannian for
  online foreground and background separation in subsampled video,'' in
  \emph{CVPR}.\hskip 1em plus 0.5em minus 0.4em\relax IEEE, 2012, pp.
  1568--1575.

\bibitem{shu2014robust}
X.~Shu, F.~Porikli, and N.~Ahuja, ``Robust orthonormal subspace learning:
  Efficient recovery of corrupted low-rank matrices,'' in \emph{CVPR}.\hskip
  1em plus 0.5em minus 0.4em\relax IEEE, 2014, pp. 3874--3881.

\bibitem{parikh2014proximal}
N.~Parikh and S.~Boyd, ``Proximal algorithms,'' \emph{Foundations and Trends in
  optimization}, vol.~1, no.~3, pp. 127--239, 2014.

\bibitem{wang2015orthogonal}
Z.~Wang, M.~J. Lai, Z.~Lu, W.~Fan, H.~Davulcu, and J.~Ye, ``Orthogonal rank-one
  matrix pursuit for low rank matrix completion,'' \emph{SIAM Journal on
  Scientific Computing}, vol.~37, no.~1, pp. A488--A514, 2015.

\bibitem{tokmakov2017learning}
P.~Tokmakov, K.~Alahari, and C.~Schmid, ``Learning video object segmentation
  with visual memory,'' in \emph{ICCV}.\hskip 1em plus 0.5em minus 0.4em\relax
  IEEE, 2017, pp. 4481--4490.

\bibitem{xu2020train}
Y.~Xu, A.~Osep, Y.~Ban, R.~Horaud, L.~Leal-Taix{\'e}, and X.~Alameda-Pineda,
  ``How to train your deep multi-object tracker,'' in \emph{CVPR}.\hskip 1em
  plus 0.5em minus 0.4em\relax IEEE, 2020, pp. 6787--6796.

\bibitem{anis2018sampling}
A.~Anis, A.~El~Gamal, A.~S. Avestimehr, and A.~Ortega, ``A sampling theory
  perspective of graph-based semi-supervised learning,'' \emph{IEEE
  Transactions on Information Theory}, vol.~65, no.~4, pp. 2322--2342, 2018.

\bibitem{kalofolias2016learn}
V.~Kalofolias, ``How to learn a graph from smooth signals,'' in \emph{AISTATS},
  2016, pp. 920--929.

\bibitem{kipf2017semi}
T.~N. Kipf and M.~Welling, ``Semi-supervised classification with graph
  convolutional networks,'' in \emph{ICLR}, 2017.

\bibitem{sandryhaila2013discrete}
A.~Sandryhaila and J.~M.~F. Moura, ``Discrete signal processing on graphs,''
  \emph{IEEE Transactions on Signal Processing}, vol.~61, no.~7, pp.
  1644--1656, 2013.

\bibitem{sandryhaila2014big}
------, ``Big data analysis with signal processing on graphs: Representation
  and processing of massive data sets with irregular structure,'' \emph{IEEE
  Signal Processing Magazine}, vol.~31, no.~5, pp. 80--90, 2014.

\bibitem{anis2014towards}
A.~Anis, A.~Gadde, and A.~Ortega, ``Towards a sampling theorem for signals on
  arbitrary graphs,'' in \emph{International Conference on Acoustics, Speech
  and Signal Processing}.\hskip 1em plus 0.5em minus 0.4em\relax IEEE, 2014,
  pp. 3864--3868.

\bibitem{shomorony2014sampling}
H.~Shomorony and A.~S. Avestimehr, ``Sampling large data on graphs,'' in
  \emph{Global Conference on Signal and Information Processing}.\hskip 1em plus
  0.5em minus 0.4em\relax IEEE, 2014, pp. 933--936.

\bibitem{chung1997spectral}
F.~R.~K. Chung, \emph{Spectral graph theory}.\hskip 1em plus 0.5em minus
  0.4em\relax American Mathematical Society, 1997, no.~92.

\bibitem{chen2016signal}
S.~Chen, R.~Varma, A.~Singh, and J.~Kova{\v{c}}evi{\'c}, ``Signal recovery on
  graphs: Fundamental limits of sampling strategies,'' \emph{IEEE Transactions
  on Signal and Information Processing over Networks}, vol.~2, no.~4, pp.
  539--554, 2016.

\bibitem{puy2016random}
G.~Puy, N.~Tremblay, R.~Gribonval, and P.~Vandergheynst, ``Random sampling of
  bandlimited signals on graphs,'' \emph{Applied and Computational Harmonic
  Analysis}, 2016.

\bibitem{marques2015sampling}
A.~G. Marques, S.~Segarra, G.~Leus, and A.~Ribeiro, ``Sampling of graph signals
  with successive local aggregations,'' \emph{IEEE Transactions on Signal
  Processing}, vol.~64, no.~7, pp. 1832--1843, 2015.

\bibitem{tsitsvero2016signals}
M.~Tsitsvero, S.~Barbarossa, and P.~Di~Lorenzo, ``Signals on graphs:
  Uncertainty principle and sampling,'' \emph{IEEE Transactions on Signal
  Processing}, vol.~64, no.~18, pp. 4845--4860, 2016.

\bibitem{romero2016kernel}
D.~Romero, M.~Ma, and G.~B. Giannakis, ``Kernel-based reconstruction of graph
  signals,'' \emph{IEEE Transactions on Signal Processing}, vol.~65, no.~3, pp.
  764--778, 2016.

\end{thebibliography}

\end{document}